%% file: main.tex
\documentclass[sigconf, nonacm]{acmart}

\setcopyright{none}

\usepackage{amsthm}
\newtheorem{proposition}{Proposition}

\usepackage{xspace}
\newcommand{\mytilde}{\textasciitilde}

\usepackage[most]{tcolorbox}
\usepackage{CJKutf8}
\newcommand{\zh}[1]{\begin{CJK*}{UTF8}{gbsn}#1\end{CJK*}}

\usepackage{enumitem}
\usepackage{multirow}
\usepackage{titletoc}
\usepackage{listings}

\begin{document}

\title{StreamProfileBench: A Benchmark for Fine-Grained User Profile Inference in Real-World Streaming Scenarios}

\author{Sizhe Wang}
\authornote{Equal Contribution.}
\authornote{The first author completed this work during an internship at DISC lab, Fudan University.}
\email{wangsz26@m.fudan.edu.cn}
\affiliation{\institution{Fudan University}\city{Shanghai}\country{China}}

\author{Feiyu Duan}
\authornotemark[1]
\email{fyduan25@m.fudan.edu.cn}
\affiliation{\institution{Fudan University}\city{Shanghai}\country{China}}

\author{Juelin Wang}
\email{jlwang26@m.fudan.edu.cn}
\affiliation{\institution{Fudan University}\city{Shanghai}\country{China}}

\author{Liwen Zhang}
\email{zhang.liwen@mail.shufe.edu.cn}
\affiliation{\institution{Shanghai University of Finance and Economics}\city{Shanghai}\country{China}}

\author{Zhongyu Wei}
\correspondingauthor
\email{zywei@fudan.edu.cn}
\affiliation{\institution{Fudan University} \institution{Shanghai Innovation Institute}\city{Shanghai}\country{China}}

\renewcommand{\shortauthors}{Wang et al.}

\begin{abstract}
Large Language Models (LLMs) have reshaped user profiling, yet current evaluations mainly focus on static data snapshots. This paradigm overlooks the reality of personalized systems, where User-Generated Content (UGC) arrives continuously and fine-grained profiles evolve rapidly. To bridge this gap, we introduce StreamProfileBench, a large-scale benchmark for fine-grained streaming user profiling. We formalize streaming user profiling as a continuous state maintenance task and curate a highly authentic dataset comprising over 120,000 UGC posts from 7,000+ real users across five diverse platforms. By leveraging the temporal correlation of user interests, we further propose a novel, annotation-free evaluation framework. Extensive experiments across 14 leading LLMs reveal that continuous profile updating remains an open challenge. Models exhibit a systemic conservative bias, over-retaining past interests while failing to recognize interest decay. Ablation experiments further validate the practical utility and necessity of the streaming paradigm. Data and code are hosted in \url{https://github.com/WaterWang-001/StreamProfileBench}.
\end{abstract}



\maketitle

\input{Section/Introduction}

\input{Section/Related_Work}
\input{Section/StreamProfileBench}

\input{Section/Experiment}
\input{Section/Discussion}
\input{Section/Conclusion}

\input{Section/Limitation}

\newpage

\bibliographystyle{ACM-Reference-Format}
\bibliography{custom}

\appendix
\input{Section/Appendix}

\end{document}

%% file: Section/Introduction.tex
\section{Introduction}

User profiling is the foundation of personalized service systems, aiming to extract attributes, interests, and intentions from personalized traces such as user-generated content (UGC) \citep{eke2019survey,prottasha2025user,tan2023user}. The rapid development of Large Language Models (LLMs) has reshaped the profiling paradigm by extending user representations from latent vectors to natural language descriptions \citep{gao2023chat,geng2022recommendation,bao2023tallrec}. Today, LLM-based user profiling has become a key online component in a wide range of complex personalized systems, including recommendation systems, conversational assistants, and role-playing agent systems \citep{liu2025recoworld,wang2023recagent,piao2025agentsociety}.


Previous studies have explored LLM-based user profiling ranging from \textit{demographic attribute profiling} \citep{li2025conf,prottasha2025user} to \textit{interest and preference profiling} \citep{lu2025prompt,sabouri2025towards,shi2025you}.
However, these efforts predominantly operate on static data snapshots. Such static settings are insufficient when user profiling serves as an online component in complex personalized systems, where two inherent challenges arise:
\begin{itemize}
    \item  \textbf{Streaming Data}: users continuously generate new content over time, making it impractical to perform profiling in a single pass over a fixed corpus.
    \item  \textbf{Evolving Profiles}: user interests are inherently hierarchical \citep{qi2021hierec,tan2023user}. While coarse-grained preferences remain largely stable, fine-grained interests shift rapidly~\citep{purificato2024user}. A profile that fails to track these fine-grained shifts quickly becomes outdated, causing degraded personalization quality in downstream systems. 
\end{itemize}

The streaming paradigm naturally addresses both challenges.                         
  Incrementally ingesting arriving content ensures \textbf{currency},                 
  directly responding to \emph{streaming data}; the persistent                        
  profile state further enables tracking interest \textbf{decay} and                  
  mapping the \textbf{narrative} of evolution, capturing the dynamics
  of \emph{evolving profiles}. Figure~\ref{fig:intro} illustrates
  these advantages over static profiling.
\begin{figure*}
    \centering
    \includegraphics[width=0.9\linewidth]{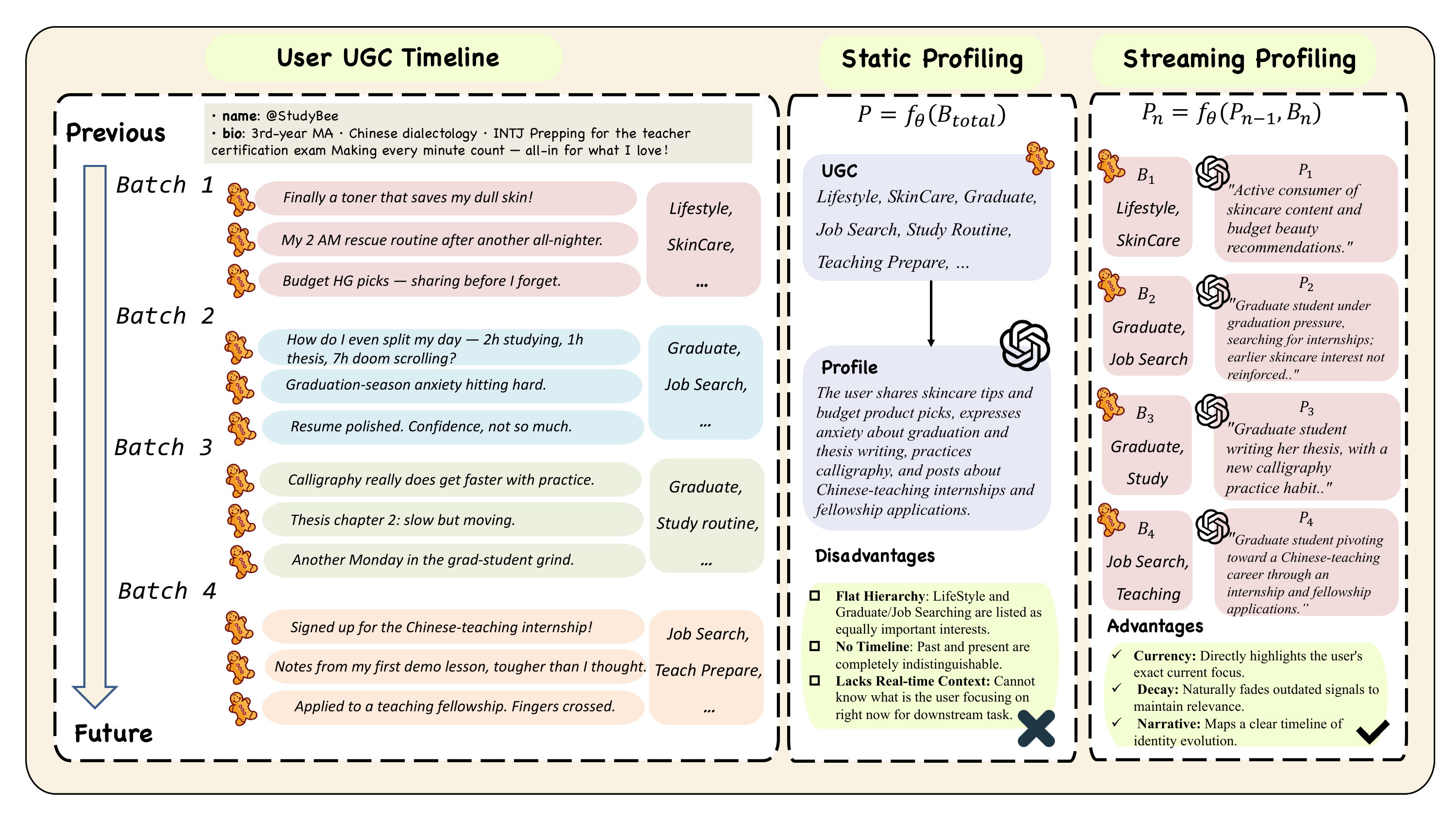}
    \caption{Illustration of differences between Static Profiling and Streaming Profiling. The streaming paradigm offers three key advantages: it highlights the user's exact current focus (Currency), naturally fades outdated signals (Decay), and maps a clear timeline of identity evolution (Narrative).}
    \label{fig:intro}
\end{figure*}



To operate effectively in such scenarios, LLMs need to both incrementally ingest arriving UGC and dynamically maintain user profiles as interests evolve, which prior user profiling evaluation work has largely overlooked \citep{wang2025lettingo,prottasha2025user}. While works such as PersonaMem \citep{jiang2025know} and Inside Out \citep{zhao2026inside} have explored maintaining evolving user profiles with LLMs, they rely on simulated dialogues and do not address the challenge of continuously profiling users from real-world UGC streams.

To fill this gap, we propose StreamProfileBench, a large-scale benchmark for fine-grained streaming user profiling. Curated through a rigorous pipeline, it structures over 120,000 real-world posts from 7,000+ users across five platforms into continuous streams via a buffering mechanism. Furthermore, we formalize streaming profiling as a state maintenance task, leveraging future interest anchors as self-verifying ground truth to rigorously assess models' ability to track temporal interest drift. 

Experiments across 14 LLMs show that streaming profile maintenance remains an open challenge. The strongest models achieves 52.26\% average recall and 54.97\% F1 recall, while weaker models only achieve 25.18\% average recall and 18.85\% F1 recall. Specifically, models struggle with novel-interest recognition and interest decay, revealing a shared inability to track how interests dynamically emerge and fade over time. Furthermore, our ablations validate the practical superiority of the streaming mechanism, demonstrating that incremental persona passing consistently outperforms both no-passing and long-context baselines.


To summarize, our contributions are threefold:
\begin{itemize}
    \item We define the problem of streaming user profiling, where LLMs must continuously track and maintain user profiles from incrementally arriving UGC, and formalize it as a state maintenance task.
    
    \item We propose StreamProfileBench, a user-centric benchmark of 7,000+ real users and 120,000+ posts across five platforms, with an evaluation framework that uses future interest anchors as self-verifying ground truth for temporal interest drift.

    \item We benchmark 14 popular LLMs. Results show that even frontier models struggle with streaming user profiling, with novel-interest recognition and interest decay as the dominant limitations, while ablations further validate the practical merit of the streaming formulation.

\end{itemize}

%% file: Section/Related_Work.tex
\section{Related Work}

\paragraph{User Profile Construction} Recent research has increasingly explored the use of LLMs to derive structured user representations from unstructured user-generated content (UGC), including tasks such as stance detection~\cite{hu2024ladder}, ideology identification~\cite{mu2024navigating}, and sentiment analysis~\cite{kheiri2023sentimentgpt}. Wu et al.~\cite{wu2024understanding} demonstrate that incorporating users' historical responses improves profiling accuracy. PersonaX~\cite{shi2025personax} employs behavior clustering to extract representative profile features, while NextQuill~\cite{zhao2025nextquill} utilizes causal inference to identify signals that reflect users' true preferences. DPL~\cite{qiu2025measuring} further models inter-user differences to capture individual-specific characteristics. However, these approaches predominantly rely on offline profile construction. In contrast, our work focuses on a streaming paradigm for user profiling.

\begin{table*}[t]
    \centering
    \small
    \begin{tabular}{lccccr}
        \toprule
        \textbf{Benchmark} 
        & \textbf{Long Horizon} 
        & \textbf{Evolving Profile} 
        & \textbf{Streaming Data} 
        & \textbf{Real User Data} 
        & \textbf{\#Users} \\
        \midrule
        LaMP~\cite{salemi2024lamp}
        & $\times$ & $\times$ & $\times$ & $\checkmark$ & N/A \\
        PersonaBench~\cite{tan2025personabench}
        & $\checkmark$ & $\times$ & $\times$ & $\times$ & 10 \\
        PrefEval~\cite{zhao2025llms}
        & $\checkmark$ & $\times$ & $\times$ & $\times$ & N/A \\
        PersonaMem-v1~\cite{jiang2025know}
        & $\checkmark$ & $\checkmark$ & $\times$ & $\times$ & 20 \\
        PersonaMem-v2~\cite{jiang2025personamem}
        & $\checkmark$ & $\checkmark$ & $\times$ & $\times$ & 1000 \\
        \midrule
        \textbf{StreamProfileBench} 
        & $\checkmark$ & $\checkmark$ & $\checkmark$ & $\checkmark$ & 7451 \\
        \bottomrule
    \end{tabular}
    \caption{Comparison of StreamProfileBench to other user profiling benchmarks.}
    \label{tab:stream_profile_bench}
\end{table*}

\paragraph{User Profile Evaluation} A range of benchmarks have been proposed to evaluate user profiling methods~\cite{bartkowiak2025edgewisepersona, liu2025itdr, ren2026alpbench}. LaMP~\cite{salemi2024lamp} models user preferences based on historical behaviors and uses prediction accuracy. PrefEval~\cite{zhao2025llms} constructs long-context conversational histories to assess models' ability to understand and infer user preferences. PersonaBench~\cite{tan2025personabench} uses synthetic privacy data to evaluate user attribute and preference extraction. However, these benchmarks are limited to static prediction, leaving long-term profile tracking largely unexplored. PersonaMem-v1~\cite{jiang2025know} investigates whether models can memorize, update, and apply evolving profiles, while PersonaMem-v2~\cite{jiang2025personamem} further incorporates implicit preferences; however, both rely on synthetic personas without grounding in real UGC. Unlike prior work, StreamProfileBench utilizes real social media data to evaluate a model’s ability to continuously infer and maintain accurate user profiles over time, offering a more realistic evaluation setting. Table~\ref{tab:stream_profile_bench} shows the comparison of profile benchmarks.

%% file: Section/StreamProfileBench.tex
\section{StreamProfileBench}
\label{sec:streamprofilebench}
Figure~\ref{fig:overview} gives an overall illustration of the design and construction of StreamProfileBench. We will introduce them in detail in this section.

\begin{figure*}[t]
    \centering
    \includegraphics[width=0.9\textwidth]{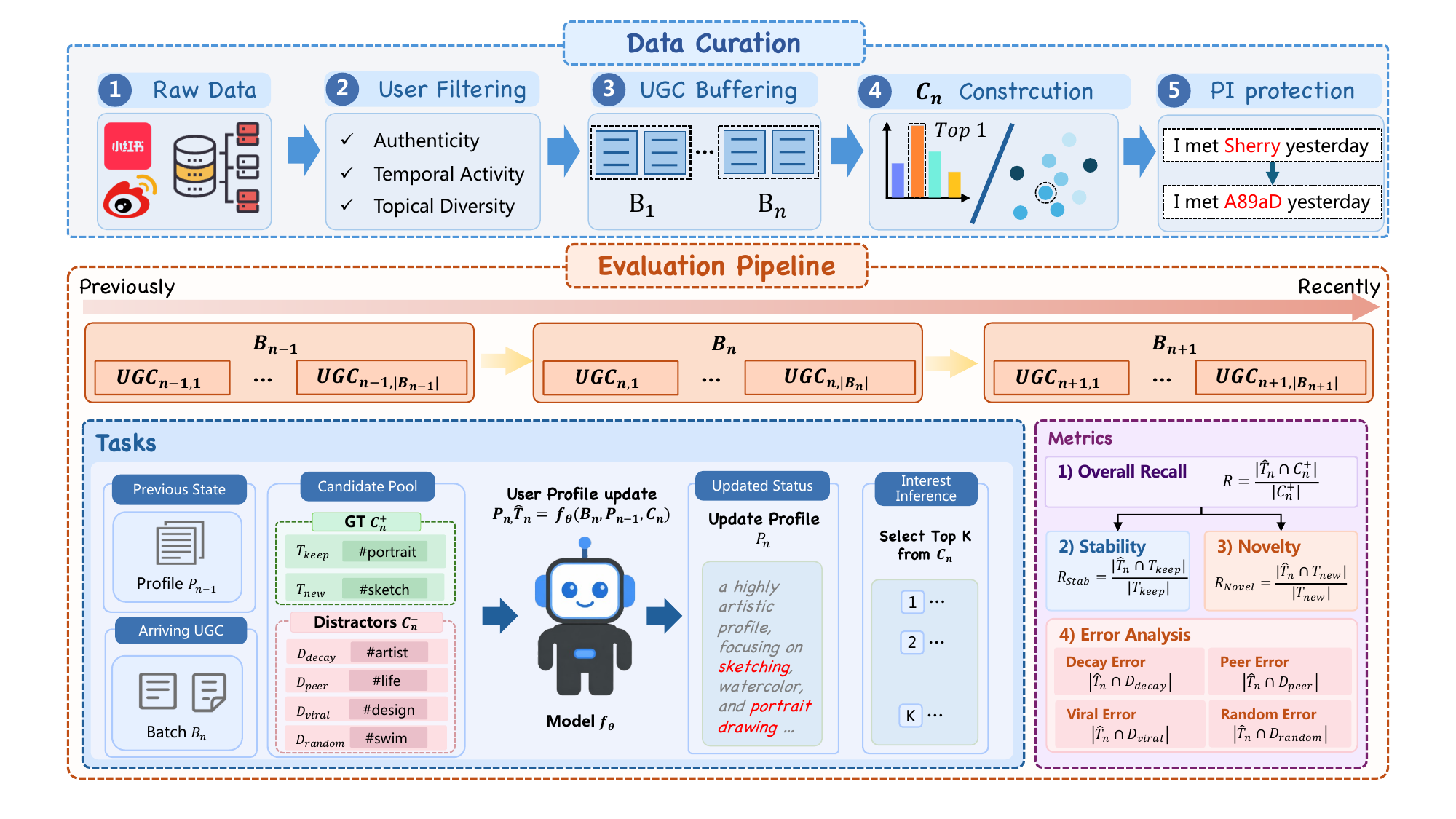}
    \caption{\textbf{Overview of StreamProfileBench.}
\textbf{Top:} The data-curation pipeline of StreamProfileBench.
\textbf{Bottom:} The Evaluation Method of StreamProfileBench.}
    \label{fig:overview}
\end{figure*}

\subsection{Task Definition}

\paragraph{General Formalization} 
Traditional user profiling is modeled as a stateless mapping. Given a user's UGC history $\mathcal{B}$, the model $f_\theta$ produce a static profile as follows:
$$\mathcal{P} = f_\theta(\mathcal{B})$$

This formulation ignores two fundamental properties of real-world UGC as mentioned in introduction. UGC in real-world arrives as an unbounded stream and user interests shift over time. We therefore formalize streaming user profiling as a \textbf{state maintenance problem}: both unbounded and non-stationarity properties demand a maintaining task form, giving rise to a read–update–write loop that evolves the profile with each incoming batch.

Formally, at each time step $n$, the model receives a new batch of UGC $\mathcal{B}_n$ from the current time window along with the user profile $\mathcal{P}_{n-1}$ carried over from the previous step. Using $\mathcal{P}_{n-1}$ as contextual prior, the model jointly reasons over the accumulated profile and the current observation $\mathcal{B}_n$ to produce an updated profile $\mathcal{P}_n$, transforming user profiling into an incremental, temporally-aware process:
$$\mathcal{P}_n = f_\theta(\mathcal{B}_n, \mathcal{P}_{n-1})$$

 \paragraph{Evaluation Formalization}   At each step, the model is required to produces a free-form natural-language profile $P_n$ that can encode any information it deems relevant, serving as its persistent memory of the user. This free-form design avoids imposing structural constraints that could interfere with evaluation. However, no reliable automatic metric exists for directly assessing $P_n$, and large-scale human annotation is prohibitively costly.

  We therefore complement $P_n$ with a structured prediction target focused on \emph{interest profiling}, which is a core dimension in streaming scenarios with clear temporal dynamics. Leveraging the temporal coherence between user interests and future UGC, we extract a ground-truth interest anchor set $T_n = \{t_1, t_2, \dots\}$ from the future batch, where each $t_i$ is a user-authored marker that explicitly manifests a proactive interest. The platform-specific extraction rules are in Appendix~\ref{app:interest_anchors}. The model is asked to infer which anchors reflect the user's future interests. by selecting from a candidate pool $C_n$ that mixes ground-truth tags in $T_n$ with distractors.

  Therefore, the evaluation task is formalized as: 
\begin{equation*}
\underbrace{(P_n,\hat{T}_n)}_{\text{updated state\& prediction}}
= f_\theta\bigl(\underbrace{P_{n-1}}_{\text{prior state}},\,
\mathcal{B}_n,\, C_n\bigr)
\end{equation*}

where each batch $\mathcal{B}_n$ is the UGC of user activities at time $n$. The candidate pool $C_n$ is partitioned into positives ($C_n^+$) and negatives ($C_n^-$). Ground truth constitutes 25\% of the pool, and the model is instructed to select exactly $|C_n^+|$ tags from $C_n$. The specific components and their definitions and evaluation targets are summarized in Table~\ref{tab:pool}. Further details are provided in Appendix~\ref{app:candidate_pool}.

\begin{table}[h]
\centering
\small
\begin{tabular}{llll}
\toprule
\textbf{Set} & \textbf{Comp.} & \textbf{Definition} & \textbf{Evaluation Target} \\
\midrule
\multirow{2}{*}{$C_n^+$} 
& $T_{\text{keep}}$ & $T_{1:n-1} \cap T_{n}$ & Interest retention \\
& $T_{\text{new}}$  & $T_{n} \setminus T_{1:n-1}$ & Emerging inference \\
\midrule
\multirow{4}{*}{$C_n^-$} 
& $D_{\text{decay}}$  & $T_{1:n-1} \setminus T_{n}$ & Decay sensitivity \\
& $D_{\text{peer}}$   & Same-cluster tags & Semantic distinction \\
& $D_{\text{viral}}$  & Trending tags & Popularity bias \\
& $D_{\text{random}}$ & Random tags & Baseline noise \\
\bottomrule
\end{tabular}
\caption{Composition of $C_n$. $C_n^+$ contains positive targets, while $C_n^-$ contains distractors derived from user history or platform trends.}
\label{tab:pool}
\end{table}

\subsection{Metric Definition}

\paragraph{Per-step Metrics}
Given the model's prediction $\hat{T}_n \subseteq C_n$ at step $n$, all metrics fundamentally evaluate the retrieval rate of a specific candidate subset $S \subseteq C_n$ using a unified scoring function:
\begin{equation*}
\mathcal{R}(S) = \frac{|\hat{T}_n \cap S|}{|S|}
\end{equation*}
Based on this formulation, we define three metric groups: \textbf{(I) Overall Recall} ($\mathcal{R}(C_n^+)$); \textbf{(II) Decomposed Recall} for stability ($\mathcal{R}(T_{\text{keep}})$) and novelty ($\mathcal{R}(T_{\text{new}})$); and \textbf{(III) Distractor Error Rates} ($\mathcal{R}(D_s)$) for distractor types $s \in \{\text{decay}, \text{peer}, \text{viral}, \text{random}\}$. Since the output budget $|C_n^+|$ is fixed, allocating capacity between $T_{\text{keep}}$ and $T_{\text{new}}$ creates an inherent \textit{stability-novelty trade-off}. The theoretical decomposition is provided in Appendix~\ref{app:tradeoff_analysis}.

\paragraph{Overall Metrics}
We aggregate step-level metrics $m_u^{(n)}$ via a two-level macro average and use overall recall $\bar{R}$ as main metric:
\begin{equation*}
\bar{M} = \frac{1}{|U|} \sum_{u \in U} \left( \frac{1}{T_u} \sum_{n=1}^{T_u} m_u^{(n)} \right)
\end{equation*}

Since $\bar{R}$ uniformly aggregates ground truth, it risks inflating scores for models biased toward abundant historical tags. To enforce balanced evaluation, we introduce $F_1^{NS}$ to strictly penalize this asymmetry:
\begin{equation*}
F_1^{NS} = \frac{2 \cdot \overline{\text{Recall}}_{\text{Nov}} \cdot \overline{\text{Recall}}_{\text{Stab}}}{\overline{\text{Recall}}_{\text{Nov}} + \overline{\text{Recall}}_{\text{Stab}}}
\end{equation*}

Formal proofs regarding the unbiasedness and convergence properties of these aggregates are provided in Appendix~\ref{app:metric_proofs}.

\subsection{Data Curation}
\label{sec:data_curation}
Our benchmark is built upon a large-scale UGC corpus spanning five major Chinese social platforms over a 30-day observation window, encompassing approximately 137 million posts and 17.4 million active users daily. Detailed statistics are provided in Appendix~\ref{app:data_source}. From this massive raw data, our pipeline with following components yields the final StreamProfileBench: a fully anonymized, high-quality user-centric dataset partitioned into temporally balanced batches. Comprehensive details for all stages are provided in Appendix~\ref{app:pipeline_details}.

\textbf{(I) User Filtering}: We target users exhibiting authentic self-expression, sustained activity, and topical diversity. Such users are selected via a two-stage pipeline: (1) a daily coarse filter that screens records at three progressive granularities using shared and platform-specific rules, and (2) a long-term stage that uses activity-based stratified sampling to retain users with appropriate active frequency. 

\textbf{(II) UGC Buffering}: Since naive calendar-day partitioning causes severe input-length imbalances due to drastically varying UGC volumes, we maintain a per-user chronological buffer. It emits a streaming step $\mathcal{B}_n$ only when the accumulated posts reach a platform-specific threshold, ensuring a uniform information load per batch.

\textbf{(III) Candidate Pool Construction}: We construct offline distractor sources: $D_\text{viral}$ samples filtered daily trending tags, while $D_\text{peer}$ retrieves semantic neighbors via an incrementally updated vector clustering index. To verify the validity of our ground truth, we conduct a human validation experiment which is described in Appendix~\ref{app:persona_eval}.
 
\textbf{(IV) Personal Information Protection}: Following standard regulations, we anonymize ten PI categories using salted hashes for identifiers and an LLM-based span detection pipeline for text, validated by manual review. Detailed information is provided in Appendix~\ref{app:pi_protection}. 

Since future anchors serve as self-verifying ground truth, their quality directly determines evaluation validity. Therefore, we also conduct a human validation experiment beyond the rule-based blacklist filtering in our pipeline. As shown in Table~\ref{tab:anchor_validation}, the overall pass rate reaches 94.7\%, confirming that the vast majority of anchors are legitimate interest signals.

\subsection{Benchmark Overview}

Table~\ref{tab:benchmark_stats} presents the overall statistics of StreamProfileBench, where 5,661 \textit{Multi} users, contribute 15,196 total prediction steps across five platforms.  Figure~\ref{fig:batch_dist} reports the distribution of batch counts per \textit{Multi} user, with the majority falling in the moderate range, capturing meaningful interest drift in observed time window. The cross-platform distribution is not artificially balanced, naturally reflecting raw data volumes and filtering yields.  
\begin{figure}[h]
    \centering
    \includegraphics[width=0.9\linewidth]{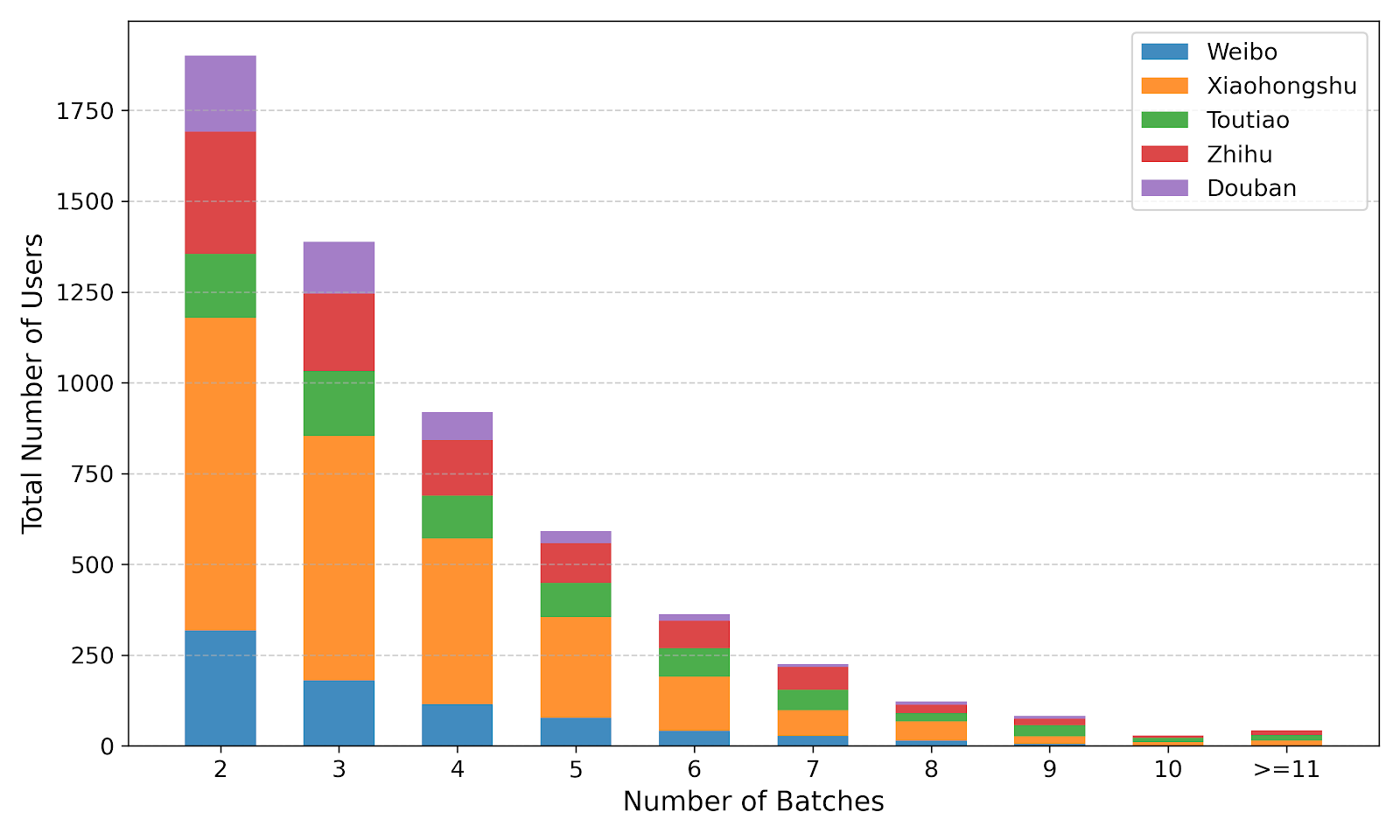}
    \caption{Distribution of batch counts per Multi user across platforms. Most users produce 2--10 streaming steps within the 30-day window.}
    \label{fig:batch_dist}
\end{figure}

We further find that benchmarked users drift far more at the anchor level
than at the coarse topic level, matching the fine-grained evaluation target of
StreamProfileBench. See Appendix~\ref{app:coarse_stability} for more details. Appendix~\ref{app:bench_example}
provides an example of one-step problem in the benchmark.

\begin{table}[h]
\centering
\small
\begin{tabular}{lrrrr}
\toprule
\textbf{Platform} & \textbf{Full} & \textbf{Multi} & \textbf{Multi\%} & \textbf{Steps} \\
\midrule
Weibo & 1,186 & 779 & 65.7\% & 1,862 \\
XHS & 3,183 & 2,585 & 81.2\% & 6,625 \\
Toutiao & 948 & 784 & 82.7\% & 2,695 \\
Zhihu & 1,414 & 1,012 & 71.6\% & 2,912 \\
Douban & 720 & 501 & 69.6\% & 1,102 \\
\midrule
\textbf{Total} & \textbf{7,451} & \textbf{5,661} & \textbf{76.0\%} & \textbf{15,196} \\
\bottomrule
\end{tabular}
\caption{Cross-platform user scale and temporal statistics of StreamProfileBench. \textit{Full}: users with completed timelines. \textit{Multi}: users with $\geq$2 steps. \textit{Steps}: total steps. XHS denotes Xiaohongshu.} 
\label{tab:benchmark_stats}
\end{table}


%% file: Section/Experiment.tex
\section{Experiment}
\subsection{Experiment Setting}

We evaluate a diverse set of 14 LLMs, covering both closed-source and open-source models across a wide range of parameter scales. Closed-source models include GPT-4o-mini \citep{hurst2024gpt}, GPT-5-mini, GPT-5.1 \citep{singh2025openai}, and Gemini-3-Flash~\citep{deepmind2025gemini3flash}. Open-source models include MiniMax-M2.5~\citep{minimax2025m25}, GLM-4.7~\citep{zaiorg2025glm47}, DeepSeek-v3.2~\citep{liu2025deepseek}, Llama-3.1-8B/70B-Instruct~\citep{grattafiori2024llama}, Qwen3-8B/14B/32B~\citep{yang2025qwen3} and GPT-oss-20B/120B~\citep{agarwal2025gpt}. More details can be found in Appendix~\ref{app:exp_config},~\ref{app:prompt_template} and~\ref{app:answer_extraction}.

\subsection{Main Result}
 Table~\ref{tab:main_results} presents the main results of all evaluated models on StreamProfileBench. We organize our findings as follows:

\begin{table*}[t]
\centering
\small
\setlength{\tabcolsep}{4pt}
\begin{tabular}{l|rrrrrr|rrrrrr}
\toprule
& \multicolumn{6}{c|}{\textbf{Average Recall ($\bar{R}$)}} & \multicolumn{6}{c}{\textbf{F1-Score of Recall ($F_1^{NS}$)}} \\
\cmidrule(lr){2-7} \cmidrule(lr){8-13}
\textbf{Model} & \textbf{Overall} & Zhihu & Weibo & Toutiao & XHS & Douban
               & \textbf{Overall} & Zhihu & Weibo & Toutiao & XHS & Douban \\
\midrule
\multicolumn{13}{l}{\textit{Closed-source Models}} \\
\midrule
GPT-4o-mini     & 33.32 & 29.97 & 34.63 & 37.42 & 34.46 & 30.13
                & 33.53 & 36.03 & 31.24 & 38.73 & 27.94 &  9.71 \\
GPT-5-mini      & 35.08 & 29.09 & 35.77 & 38.47 & 38.62 & 33.45
                & 35.44 & 35.70 & 35.97 & 36.87 & 29.02 & \underline{33.81} \\
GPT-5.1         & 38.87 & 44.52 & 38.65 & 41.50 & \underline{41.69} & 28.01
                & 42.73 & \textbf{52.63} & 41.97 & \underline{43.05} & \underline{38.75} & 28.88 \\
Gemini-3-Flash  & \textbf{52.26} & \textbf{62.95} & \textbf{46.76} & \textbf{50.09} & \textbf{48.37} & \textbf{53.15}
                & \textbf{54.97} & \underline{49.30} & \textbf{52.75} & \textbf{53.71} & \textbf{48.78} & \textbf{52.65} \\
\midrule
\multicolumn{13}{l}{\textit{Open-source Models}} \\
\midrule
MiniMax-M2.5    & 35.61 & 36.90 & 35.53 & 38.06 & 38.37 & 29.17
                & 37.90 & 41.56 & 35.79 & 39.09 & 31.26 & 26.18 \\
GLM-4.7         & 41.69 & 49.56 & \underline{40.03} & 39.49 & 40.73 & 38.65
                & 44.21 & 47.79 & \underline{43.62} & 40.24 & 35.72 & 26.10 \\
DeepSeek-v3.2   & \underline{43.05} & \underline{51.50} & 39.63 & \underline{42.65} & 40.30 & \underline{41.19}
                & \underline{44.63} & 48.47 & 40.40 & 41.44 & 33.59 & 31.12 \\
Llama-3.1-8B   & 25.18 & 17.36 & 29.36 & 30.68 & 30.81 & 17.69
                    & 18.85 & 19.43 & 21.21 & 20.29 & 14.52 & 18.80 \\
Llama-3.1-70B  & 38.08 & 42.43 & 38.07 & 38.96 & 36.36 & 34.58
                    & 40.38 & 42.76 & 40.46 & 40.80 & 32.91 & 17.55 \\
Qwen3-8B        & 30.97 & 25.04 & 34.10 & 37.52 & 36.09 & 22.10
                & 31.54 & 28.99 & 32.09 & 36.89 & 29.37 & 25.05 \\
Qwen3-14B       & 39.44 & 48.93 & 37.00 & 40.09 & 37.30 & 33.88
                & 41.33 & 48.68 & 37.25 & 39.68 & 30.37 & 28.64 \\
Qwen3-32B       & 38.26 & 45.65 & 37.26 & 40.54 & 37.37 & 30.46
                & 40.82 & 48.21 & 38.54 & 41.30 & 32.39 & 27.43 \\
GPT-oss-20B     & 34.08 & 26.66 & 34.36 & 35.73 & 35.88 & 37.75
                & 34.68 & 32.07 & 31.78 & 32.43 & 27.21 & 29.53 \\
GPT-oss-120B    & 35.25 & 31.64 & 35.76 & 37.36 & 36.57 & 34.90
                & 37.12 & 39.17 & 35.09 & 36.77 & 30.22 & 33.68 \\
\bottomrule
\end{tabular}
\caption{Main results on StreamProfileBench across five Chinese social media platforms.
We report the macro-averaged overall Recall ($\bar{R}$) and the novelty-stability
balance score ($F_1^{NS}$) per platform. \textbf{Best} and \underline{second-best} results are highlighted.}
\label{tab:main_results}
\end{table*}

\textbf{Model Capabilities: Open-Source Competitiveness and Scaling Limits.} While closed-source models like Gemini-3-Flash achieve peak performance, state-of-the-art open-source architectures effectively close the gap, with models like DeepSeek-v3.2 and GLM-4.7 even surpassing GPT-5.1 on overall recall. Furthermore, while increasing model size generally yields initial benefits, these gains diminish at larger scales. For instance, Qwen3-32B performs comparably to Qwen3-14B, suggesting that beyond a certain threshold, raw parameter count is no longer the primary bottleneck for complex, continuous profiling tasks.

\textbf{$\bar{R}$ vs.\ $F_1^{NS}$: Coverage vs. Balance.} 
Rankings of models like GPT-5.1 and DeepSeek-v3.2 under $\bar{R}$ and $F_1^{NS}$ are not aligned. These divergences confirm that $\bar{R}$ rewards overall coverage while $F_1^{NS}$ penalizes stability-novelty imbalance; reporting both provides a multi-dimensional perspective on model performance. Gemini-3-Flash is the only model that leads on both metrics, showing that high overall recall and balanced allocation are not mutually exclusive.

\textbf{Cross-Platform Variance: No Universal Winner.} Models exhibit pronounced specialization across platforms. Gemini-3-Flash leads on Weibo, Toutiao, and Xiaohongshu but underperforms on Zhihu and Douban. This divergence highlights the inherent differences in user demographics, community cultures and diverse UGC formats across ecosystems, showing that streaming user profiling is not a monolithic capability but a composition of platform-specific competencies.

\textbf{Non-LLM Baselines: Genuine Reasoning Beyond Pattern Matching.} To verify that LLMs perform meaningful profile reasoning beyond simple pattern matching, we design two groups of non-LLM baselines with increasing information access. As Table~\ref{tab:nonllm_baselines} shows, all non-LLM baselines fall short of LLM method. This confirms that LLMs contribute genuine profile reasoning beyond lexical or embedding-based pattern matching.

\begin{table}[h]
\centering
\small
\setlength{\tabcolsep}{4pt}
\begin{tabular}{lrrrr}
\toprule
\textbf{Method} & \textbf{Recall} & \textbf{R$_{\text{Nov}}$} & \textbf{R$_{\text{Stab}}$} & \textbf{F$_1^{NS}$} \\
\midrule
Random            & 23.7 & 23.7 & 23.5 & 14.1 \\
Last-Step Copy    & 19.8 & 4.9  & 86.3 & 8.3  \\
Hist. Frequency   & 22.0 & 7.4  & 87.8 & 14.0 \\
Recency-Weighted  & 21.5 & 6.7  & 88.1 & 12.6 \\
TF-IDF            & 29.4 & 17.2 & \textbf{89.2} & 29.0 \\
Embedding Sim.    & 33.6 & 26.6 & 69.2 & 31.7 \\
\midrule
DeepSeek-v3.2     & \textbf{43.1} & \textbf{33.6} & 67.6 & \textbf{44.6} \\
\bottomrule
\end{tabular}
\caption{Non-LLM baselines vs. a representative LLM. History-signal baselines use only tag metadata; content-matching baselines read post text without LLM reasoning.}
\label{tab:nonllm_baselines}
\end{table}

Besides, since the free-form persona $P_n$ is generated but not directly scored, we conduct a human evaluation of final-step personas along five dimensions on three models spanning different capability tiers. Details can be found in Appendix~\ref{app:persona_eval}.

\subsection{Stability-Novelty Trade-off}

To understand the inherent factors influencing model performance, we decompose the overall recall into its two constituent dimensions: $\text{Recall}_{\text{Stability}}$ and $\text{Recall}_{\text{Novelty}}$. Table~\ref{tab:tradeoff} presents the overall breakdown of both metrics across all models. 

\begin{table*}[t]                                                                                  
  \centering                                                                                         
  \small                                                                                             
  \setlength{\tabcolsep}{4pt}
  \begin{tabular}{l|rrrrrr|rrrrrr}                                                                   
  \toprule        
  & \multicolumn{6}{c|}{\textbf{Recall$_{\text{Stability}}$}} &                                      
  \multicolumn{6}{c}{\textbf{Recall$_{\text{Novelty}}$}} \\                                          
  \cmidrule(lr){2-7} \cmidrule(lr){8-13}
  \textbf{Model} & \textbf{Overall} & Zhihu & Weibo & Toutiao & XHS & Douban & \textbf{Overall} & Zhihu &
  Weibo & Toutiao & XHS & Douban \\
  \midrule
  \multicolumn{13}{l}{\textit{Closed-source Models}} \\
  \midrule
  GPT-4o-mini      & 67.28 & 70.84 & 87.99 & 83.24 & 80.15 & 14.17 & 23.34 & 25.19 & 19.66 & 23.73 &
  17.32 & 30.81 \\
  GPT-5-mini       & \textbf{82.27} & \underline{87.93} & 90.03 & \textbf{90.77} & \textbf{89.65} &
  \underline{52.99} & 23.18 & 21.93 & 21.19 & 22.54 & 18.46 & 31.79 \\
  GPT-5.1          & 71.21 & 66.96 & 90.47 & 83.18 & 85.48 & 29.95 & 30.53 & 43.34 & 27.32 & 29.04 &
  25.06 & 27.88 \\
  Gemini-3-Flash   & 70.17 & 40.95 & 83.89 & 83.59 & 86.32 & \textbf{56.09} & \textbf{45.83} &
  \textbf{65.44} & \underline{37.87} & \textbf{39.72} & 33.67 & \textbf{52.44} \\
  \midrule
  \multicolumn{13}{l}{\textit{Open-source Models}} \\
  \midrule
  MiniMax-M2.5        & 67.49 & 48.81 & 89.69 & \underline{87.26} & 87.83 & 23.88 & 35.61 & 36.90 &
  35.53 & 38.06 & \underline{38.37} & 29.17 \\
  GLM-4.7             & 65.99 & 45.98 & \textbf{91.57} & 85.12 & 87.69 & 19.61 & \underline{41.69} &
  49.56 & \textbf{40.03} & \underline{39.49} & \textbf{40.73} & 38.65 \\
  DeepSeek-v3.2       & 67.60 & 46.15 & \underline{90.86} & 87.24 & \underline{88.89} & 24.86 & 33.49
   & \underline{51.46} & 25.98 & 27.68 & 20.72 & \underline{41.59} \\
  Llama-3.1-8B   & 70.94 & 73.25 & 87.25 & 87.24 & 86.39 & 20.59 & 12.00 & 11.20 & 12.07 & 11.48
   & 7.93 & 17.30 \\
  Llama-3.1-70B  & 58.84 & 42.51 & 81.29 & 78.83 & 79.90 & 11.66 & 30.74 & 43.00 & 26.93 & 27.53
   & 20.72 & 35.51 \\
  Qwen3-8B            & \underline{74.00} & \textbf{88.43} & 84.83 & 83.79 & 83.04 & 29.90 & 20.04 &
  17.34 & 19.79 & 23.66 & 17.84 & 21.56 \\
  Qwen3-14B           & 65.84 & 49.25 & 85.49 & 85.34 & 84.54 & 24.59 & 30.12 & 48.13 & 23.81 & 25.85
   & 18.51 & 34.30 \\
  Qwen3-32B           & 65.95 & 51.78 & 84.94 & 84.60 & 83.38 & 25.06 & 29.55 & 45.11 & 24.93 & 27.32
   & 20.10 & 30.31 \\
  GPT-oss-20B         & 72.99 & 81.27 & 87.66 & 87.24 & 84.70 & 24.06 & 22.75 & 19.97 & 19.41 & 19.92
   & 16.21 & 38.22 \\
  GPT-oss-120B        & 72.25 & 73.29 & 87.71 & 83.78 & 83.31 & 33.15 & 24.98 & 26.72 & 21.93 & 23.55
   & 18.46 & 34.23 \\
  \bottomrule
  \end{tabular}
  \caption{Decomposition of recall into stability ($T_{\text{keep}}$) and novelty ($T_{\text{new}}$)
  dimensions. \textbf{Best} and \underline{second-best} results are highlighted.}
  \label{tab:tradeoff}
  \end{table*}

\textbf{Models favor established anchors over emerging signals.} Table~\ref{tab:tradeoff} reveals a consistent asymmetry: $\overline{\text{Recall}}_{\text{Stability}}$ substantially exceeds $\overline{\text{Recall}}_{\text{Novelty}}$ across all models. LLMs over-weight the past while under-reacting to the present. Instead of proactively tracking dynamic shifts, they tend to choose high-frequency historical tags. This conservative bias serves as the main challenge of models' overall performance.

\textbf{External error and internal balance are decoupled.} Following Appendix~\ref{app:tradeoff_analysis}, we decompose model behavior into an \emph{external} constraint budget and an \emph{internal} recall ratio. Figure~\ref{fig:tradeoff_axes} reveals three regimes: \textbf{Gemini-3-Flash} dominates with a high budget and balanced ratio; \textbf{GPT-5-mini} secures a comparable budget but skews heavily toward stability; and \textbf{Llama-3.1-8B} lags on both axes, indicating external noise must be addressed before internal rebalancing can yield gains.

\textbf{Navigating the trade-off for an optimal $F_1^{NS}$.} 
Achieving an optimal $F_1^{NS}$ requires mitigating distractor confusion and balancing the stability-novelty allocation. Figure~\ref{fig:constraint_lines} illustrates this geometrically: Gemini-3-Flash and GLM-4.7 sit close to their optimal tangent points. Conversely, GPT-5-mini is pushed toward the stability-dominant limit by extreme conservative bias, while Llama-3.1-8B's constraint line is severely restricted by excessive distractor errors.

\begin{figure}[t]
    \centering
    \includegraphics[width=1.0\linewidth]{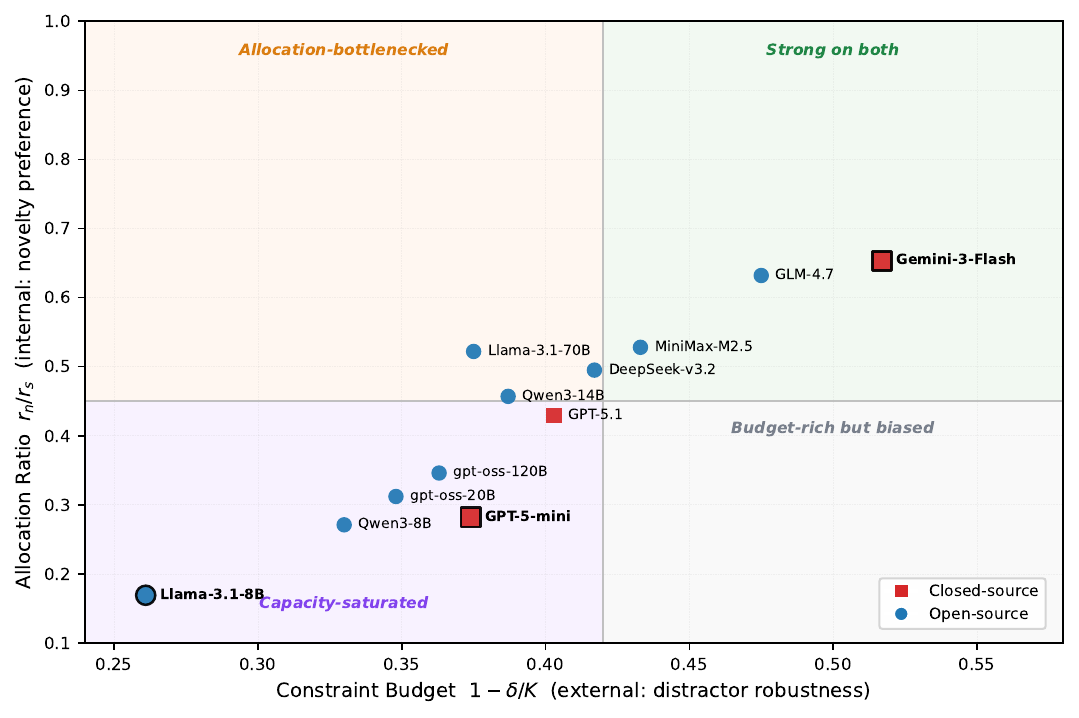}
    \caption{Two-axis decomposition of streaming profiling capability. The $x$-axis is the external constraint budget $B = 1-\delta/K$ (distractor robustness) and the $y$-axis is the internal recall ratio $\rho = r_n/r_s$ (implicit allocation preference).}
    \label{fig:tradeoff_axes}
\end{figure}

\begin{figure}[t]
    \centering
    \includegraphics[width=1.0\linewidth]{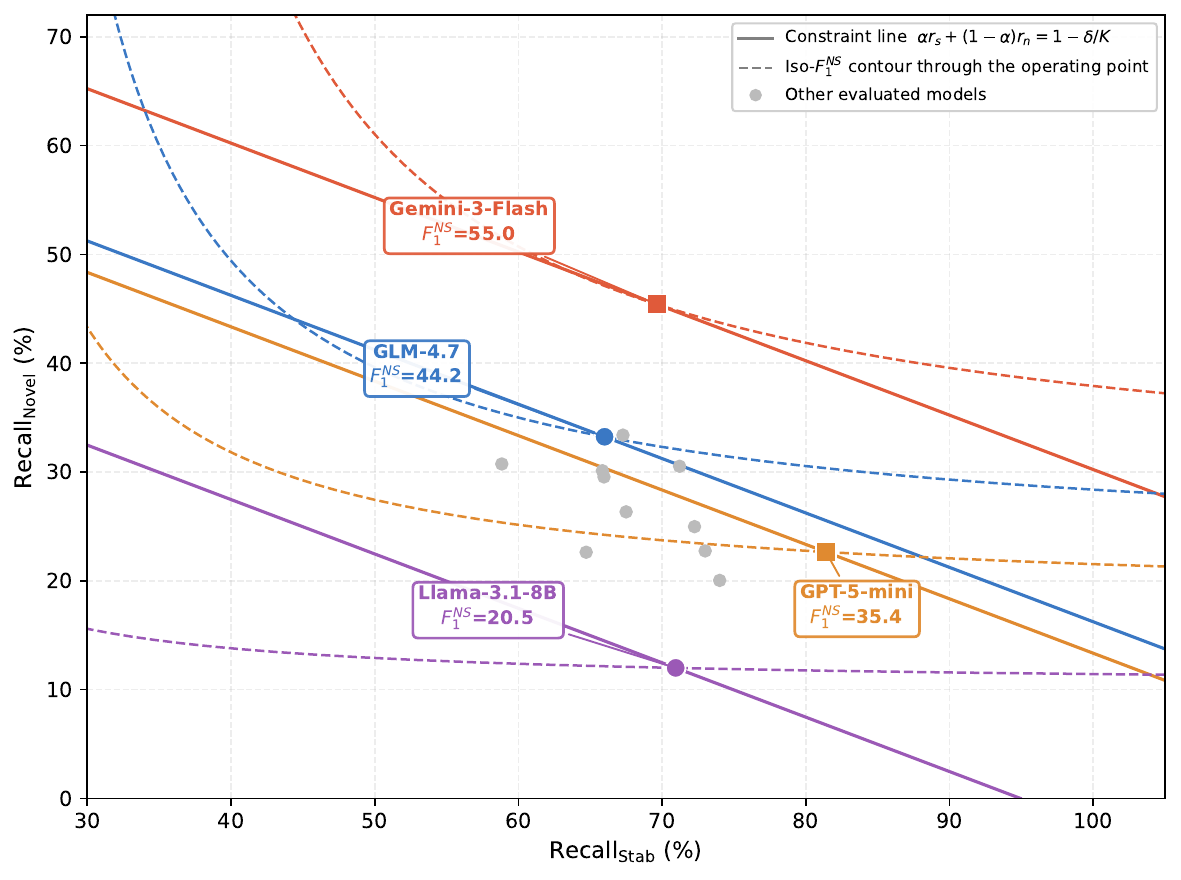}
    \caption{Constraint lines and iso-$F_1^{NS}$ contours of four representative models in the $(\mathrm{Recall}_{\mathrm{Stab}}, \mathrm{Recall}_{\mathrm{Novel}})$ plane under $\alpha=0.24$. Solid lines denote each model's constraint; dashed curves are the iso-$F_1^{NS}$ contours passing through its operating point. Other evaluated models are shown for reference.}
    \label{fig:constraint_lines}
\end{figure}

\subsection{Error Analysis}

To identify the source of distractor confusion $\delta$, we decompose it into four error rates corresponding to the four distractor types. Table~\ref{tab:error_compact} reports the platform-averaged error rates for each evaluated model; the full per-platform breakdown is deferred to Appendix~\ref{app:error_full}. Main findings are as follows:

\begin{table}[h]
\centering
\small
\setlength{\tabcolsep}{4pt}
\begin{tabular}{lrrrr}
\toprule
\textbf{Model} & $E_D$ & $E_P$ & $E_V$ & $E_R$ \\
\midrule
GPT-4o-mini      & 65.8 & 4.6 & 6.1 & 5.7 \\
GPT-5-mini       & 71.9 & 4.6 & \textbf{2.8} & \textbf{3.3} \\
GPT-5.1          & 62.6 & 6.1 & \underline{3.6} & \underline{3.5} \\
Gemini-3-Flash   & \textbf{48.0} & 9.2 & 4.6 & 5.7 \\
MiniMax-M2.5     & 65.6 & 5.5 & 3.7 & 3.8 \\
GLM-4.7          & 58.2 & 5.8 & 5.1 & 5.4 \\
DeepSeek-v3.2    & 58.3 & 6.8 & 4.5 & 5.2 \\
Llama-3.1-8B     & 71.7 & \textbf{3.2} & 4.7 & 4.5 \\
Llama-3.1-70B    & \underline{51.5} & 9.0 & 7.9 & 8.2 \\
Qwen3-8B         & 70.4 & \underline{4.5} & 5.5 & 3.9 \\
Qwen3-14B        & 58.9 & 7.1 & 5.0 & 5.2 \\
Qwen3-32B        & 60.8 & 6.9 & 5.3 & 5.7 \\
GPT-oss-20B      & 68.3 & 4.6 & 4.5 & 5.0 \\
GPT-oss-120B     & 66.1 & 6.3 & 4.8 & 4.9 \\
\bottomrule
\end{tabular}
\caption{Platform-averaged distractor error rates (\%). $E_D$: decayed interests; $E_P$: semantic peers; $E_V$: viral trends; $E_R$: random noise. Lower is better.}
\label{tab:error_compact}
\end{table}

\textbf{Decay errors dominate.} Across all models, $E_D$ is significantly larger than the other three error rates. The ranking of models by $E_D$ closely mirrors their overall $\bar{R}$ and $F_1^{NS}$ rankings. This confirms that the primary source of distractor leakage is the inability to distinguish recently active but now abandoned interests from currently sustained ones. Temporal sensitivity to interest decay is still the bottleneck for streaming user profiling. 

\textbf{Other errors are uniformly low.} $E_P$, $E_V$, and $E_R$ all stay below $10$\% for nearly every model and exhibit only weak model dependence, indicating that current LLMs handle coarse semantic discrimination and noise rejection reasonably well. 

%% file: Section/Discussion.tex
\section{Discussion}
\label{sec:discussion}

Beyond the primary benchmark rankings, we aim to validate the practical utility of the streaming paradigm and the soundness of our evaluation design. Through targeted studies across three research questions, we demonstrate that the streaming mechanism fundamentally outperforms both no-passing and uncompressed long-context baselines, and that our structured interest metrics align with human-judged persona quality. Details are presented below and in Appendix~\ref{app:ablation}.

\paragraph{RQ1: Does persona passing indeed benefit models in streaming profile inference?} 

 While persona passing is intuitively appealing for retaining accumulated knowledge, it is not guaranteed that the passed persona consistently provides \emph{useful} signals rather than accumulating \emph{noise}. Therefore, we conduct an ablation study comparing the \texttt{full\_persona} against a \texttt{no\_persona} baseline.

\begin{figure}
    \centering
    \includegraphics[width=0.85\linewidth]{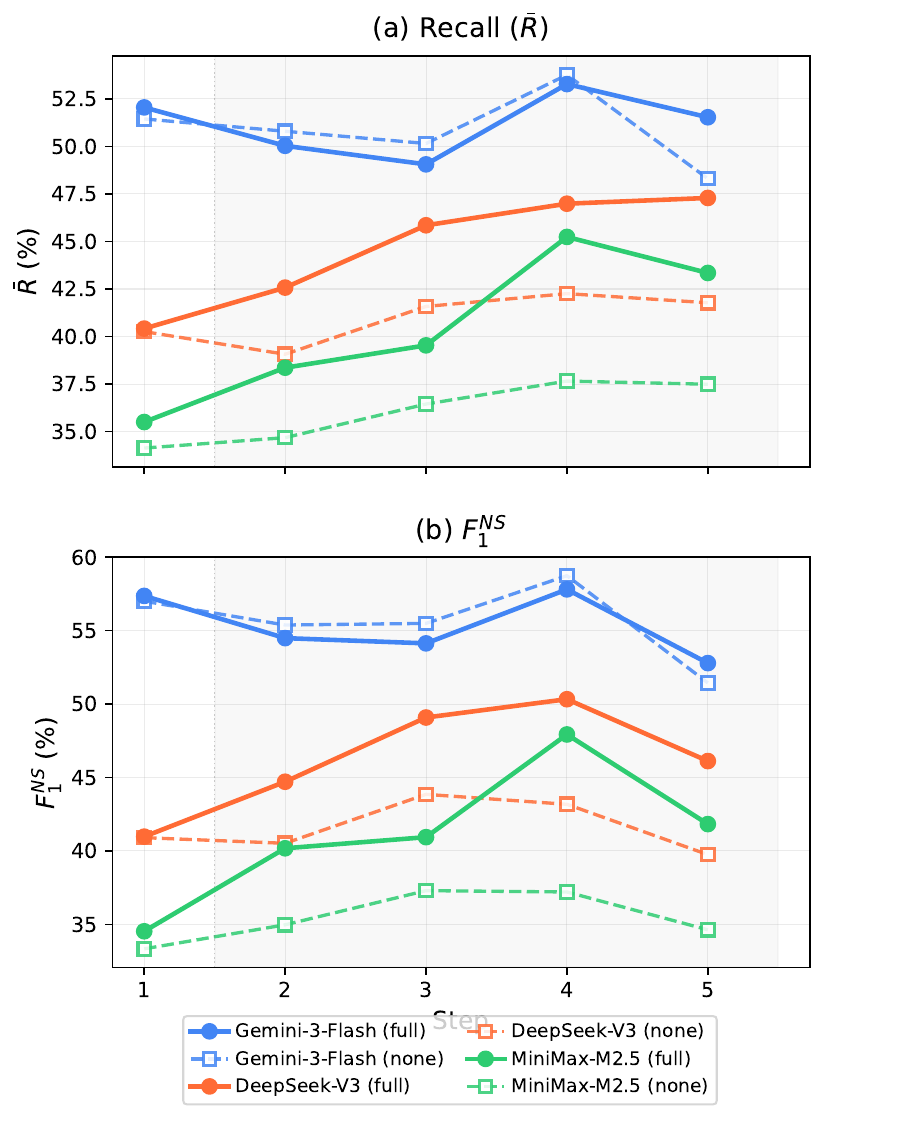}
    \caption{Step-wise performance under persona passing ablation. Solid lines: \texttt{full}; dashed lines: \texttt{none}.}
    \label{fig:persona_stepwise}
\end{figure}

Overall, \textbf{persona passing yields clear gains for moderate models, yet exposes a critical bottleneck in streaming state maintenance for frontier models.} Effective streaming profiling requires handling different input sources, namely both comprehending the current UGC context and maintaining the profile state. As Figure~\ref{fig:persona_stepwise} shows, DeepSeek-v3.2 and MiniMax-M2.5 exhibit a widening gap between \texttt{full} and \texttt{none} from Step 2 onward, relying on explicit historical guidance to boost performance. However, Gemini-3-Flash's curves overlap or invert slightly. This indicates that for leading models, comprehension of current batch is already highly proficient. Their lack of gain from persona passing highlights that even the strongest models fail to effectively maintain and fully utilize historical profile information over a continuous stream.

\paragraph{RQ2: Can massive context windows replace incremental persona           
  compression?}                                                                                      
 Streaming profiling compresses raw posts into iteratively refined persona
  representations. Since modern LLMs can accommodate a user's long-horizon post history within their
  context windows, a natural question arises: does iterative persona compression
  outperform long-context? We vary the batching granularity
   from fine to very coarse and compare against a
  \texttt{long-context} baseline that processes all posts at once without
  compression.

\begin{table}[t]
  \centering
  \small
  \begin{tabular}{l|ccc|c}
    \toprule
    & \multicolumn{3}{c|}{Streaming} & \\
    \textbf{Model} & Fine & Default & Coarse & LongCtx \\
    \midrule
    Gemini-3-Flash & \textbf{56.1} & 53.6 & 50.2 & 53.7 \\
    DeepSeek-v3.2  & \textbf{50.0} & 47.5 & 45.7 & 39.1 \\
    MiniMax-M2.5   & \textbf{43.6} & 43.1 & 38.4 & 36.2 \\
    \midrule
    Gemini-3-Flash & \textbf{52.6} & 49.5 & 47.7 & 49.9 \\
    DeepSeek-v3.2  & \textbf{47.6} & 44.4 & 44.5 & 38.9 \\
    MiniMax-M2.5   & \textbf{42.9} & 40.9 & 39.5 & 36.6 \\
    \bottomrule
  \end{tabular}
  \caption{Performance across streaming granularities vs.\ long-context baseline. Top: $F_1^{NS}$ (\%), Bottom: $\bar{R}$ (\%).}
  \label{tab:streaming_vs_longctx}
\end{table}

  As Table~\ref{tab:streaming_vs_longctx} shows, \textbf{streaming performance scales
  monotonically with iteration count across all models}. The long-context
  baseline serves as a reference point
  for when persona compression becomes worthwhile. The crossover point at which streaming surpasses long-context is model-dependent.  This reveals a consistent pattern: beyond its necessity for handling unbounded streams, streaming with sufficient    
  iterations is mechanistically superior to long-context even when all data fits    
  within the context window, where iterative refinement progressively distills signals that
   one-shot processing over raw data cannot extract.

%% file: Section/Conclusion.tex
\section{Conclusion}

In this work, we present StreamProfileBench, the first benchmark for fine-grained streaming user profiling in real-world scenarios, formalizing this challenge as a state maintenance task equipped with a temporal self-verifying mechanism. Our benchmark comprises over 120,000 UGC posts from 7,000+ real users across five diverse platforms. Evaluation of 14 LLMs reveals that models struggle with the streaming profiling task, conservatively preferring historical traits while remaining highly vulnerable to interest decay. Furthermore, ablation
  experiments validate the practical necessity of the streaming paradigm.
  Overall, this work highlights critical capability gaps of current LLMs
  in streaming profiling, establishing a foundation for more robust
  long-horizon personalized agent systems.

%% file: Section/Limitation.tex


\section*{Ethical Considerations}

We carefully consider the ethical implications of constructing and releasing this benchmark.

\begin{itemize}
    \item \textbf{Data sources.} All records are collected from publicly accessible UGC on mainstream Chinese social media platforms; no private messages, protected accounts, non-public APIs, or purchased data are involved.

    \item \textbf{Personal information protection.} The entire corpus is processed through the anonymization pipeline in Section~\ref{sec:data_curation} before release, so no directly identifiable personal information is exposed.

    \item \textbf{Intended use.} The benchmark is released under a research-only license for academic research on streaming user interest modeling; commercial profiling, targeted advertising, individual-level tracking, and any re-identification attempt are explicitly forbidden.

    \item \textbf{Representativeness and biases.} The benchmark reflects the user distribution of the five collected platforms; results should not be extrapolated to other languages, platforms, or populations without validation.
\end{itemize}

We commit to updating the released data promptly based on community feedback, including any takedown request from identifiable users.

%% file: Section/Appendix.tex
\section*{Appendix}

\label{sec:appendix}
\section{Task Definition}
\subsection{Interest Anchor Extraction}
\label{app:interest_anchors}

This section expands on the notion of \textit{interest anchors}. An interest anchor is a user-authored marker embedded in a post that signals deliberate topical intent---a hashtag, an explicit question, a bookmarking action, or any other discrete token whose presence indicates the user has consciously committed to a topic rather than mentioning it in passing. Because anchors are written by the user themselves, they provide a ground-truth signal that is both fine-grained and free of human annotation, and they form the basis of $T_\text{keep}$, $T_\text{new}$, and the distractor components $D_\text{peer}$ and $D_\text{viral}$ in the candidate pool $C_n$.

\paragraph{Anchor sources across platforms.}
Each of the five platforms surfaces anchors through a different native mechanism, reflecting differences in content format, community conventions, and UI affordances. We therefore implement a separate extractor per platform, unified by a shared post-processing cleaner (length bounds, blacklist filtering, density constraints; see Appendix~\ref{app:pipeline_details}). Table~\ref{tab:anchor_examples} lists the extraction rule, a representative example, and the accepted tag-density interval for each platform. 

\begin{table*}[h]
\centering
\small
\begin{tabular}{lp{5cm}lc}
\toprule
\textbf{Platform} & \textbf{Extraction rule} & \textbf{Example} & \textbf{Density interval} \\
\midrule
Weibo       & Double-hash topic tag \texttt{\#...\#}                                   & \zh{\#高考作文\#}          & $[0.2, 1.0]$ \\
Xiaohongshu & Single-hash topic tag \texttt{\#...}                                     & \zh{\#咖啡探店}            & $[1.0, 4.0]$ \\
Toutiao     & Single-hash topic tag \texttt{\#...}      & \zh{\#三农}                & $[0.2, 3.0]$ \\
Zhihu       & Question title $+$ TF-IDF top-5 keywords                           & \zh{``如何看待\ldots''}    & $[0.2, 2.0]$ \\
Douban      & item link $+$ action verb $\times$ domain type & \zh{看过$\cdot$《寄生虫》} & $[0.2, 2.0]$ \\
\bottomrule
\end{tabular}
\caption{Per-platform interest-anchor extraction strategies, representative examples, and the admissible tag-density interval used as a user-level sanity check.}
\label{tab:anchor_examples}
\end{table*}

\paragraph{Why these sources qualify as anchors.}
The common property of all five extraction rules is that the resulting tokens are \textit{intent-bearing}: using a hashtag on Weibo or Xiaohongshu is a voluntary act of joining a topic stream, a Zhihu question title is the explicit framing chosen by the asker, and a Douban ``watched / want-to-read'' marker is a deliberate bookmarking action tied to a concrete work. This distinguishes anchors from ordinary content tokens, which may appear by chance or as part of a quoted passage and therefore carry far noisier supervision. 

\paragraph{Normalization.}
Before being used as ground truth or entering the distractor pool, raw anchors go through a shared normalization step. The details of this cleanup pipeline, and of how the resulting anchors flow into $D_\text{viral}$ and $D_\text{peer}$, are given in Appendix~\ref{app:pipeline_details}.

\subsection{Candidate Pool Construction}
\label{app:candidate_pool}

This section expands on the candidate pool $C_n$. $C_n$ is the closed tag set from which the model must pick exactly $|C_n^+|$ items at step $n$; its composition is what turns a free-form generation task into a reproducible, annotation-free prediction problem. We therefore design $C_n$ with two goals: (i) the positive side $C_n^+$ must reflect the user's \textit{actual} next-step behaviour so that correctness is automatically decidable from future anchors, and (ii) the negative side $C_n^-$ must be structured enough to diagnose specific failure modes, not merely to lower random-guess accuracy.

\paragraph{Size and positive ratio.}
At every step we fix the positive proportion to $|C_n^+| / |C_n| = 25\%$ across all platforms. This ratio is small enough that trivial retrieval baselines cannot saturate the metric, yet large enough that the per-step metric variance stays manageable at typical $|C_n^+|$ values of 2--7. The model is instructed to output exactly $|C_n^+|$ tags, which keeps Precision and Recall identical on a per-step basis and removes the need to tune a decision threshold.

\paragraph{Positive pool.}
$C_n^+$ is drawn from the \textit{future} step $\mathcal{B}_{n+1}$ and is partitioned into two subsets according to the user's own history:
\begin{itemize}
    \item $T_\text{keep}$ collects anchors that already appeared in $\mathcal{B}_{1:n}$ and appear \emph{again} in $\mathcal{B}_{n+1}$. These test the model's ability to \textbf{consolidate long-term interests}: recognising that an established preference persists across steps.
    \item $T_\text{new}$ collects anchors that were \emph{absent} from $\mathcal{B}_{1:n}$ but appear for the first time in $\mathcal{B}_{n+1}$. These test the model's ability to \textbf{generalize to emerging interests}: inferring novel topics that are consistent with the user's underlying profile but have not yet been observed.
\end{itemize}
Because both subsets come from the user's own post stream rather than from human annotation, the ground truth is both noise-free (the user literally wrote those anchors) and cost-free (no labellers needed).

\paragraph{Negative pool.}
$C_n^-$ is a structured set of four distractor types, each probing a distinct axis of model failure:
\begin{itemize}
    \item $D_\text{decay}$ sampled from high-frequency anchors in the recent history $\mathcal{B}_{1:n}$ that \emph{do not} reappear in $\mathcal{B}_{n+1}$. Selecting from $D_\text{decay}$ indicates the model fails to detect \textbf{interest decay} and blindly projects historical topics forward.
    \item $D_\text{peer}$ sampled from the semantic cluster of a true positive tag (using the clustering index described in Appendix~\ref{app:pipeline_details}), excluding tags the user has actually engaged with. Selecting from $D_\text{peer}$ indicates the model has correctly identified the user's broad topical domain but lacks \textbf{fine-grained discrimination} within that domain.
    \item $D_\text{viral}$ sampled from the daily platform-wide trending list of the target day (coverage statistics in Appendix~\ref{app:pipeline_details}), excluding any tag related to the user's profile. Selecting from $D_\text{viral}$ indicates susceptibility to \textbf{popularity bias}---letting "what everyone is talking about today" override personal preference.
    \item $D_\text{random}$ sampled uniformly from the global tag pool. This is the weakest distractor class and serves as a baseline: a model that selects from $D_\text{random}$ has essentially no personalised signal at all.
\end{itemize}
The four types are complementary. A model that scores well on the overall Recall but incurs high error on $D_\text{peer}$, for example, has learned the user's macro topic but not their micro preferences; a model that incurs high error on $D_\text{decay}$ has learned the topics but not their temporal dynamics. Reporting per-distractor error rates therefore converts a single aggregate score into a multi-dimensional failure profile.

\paragraph{Sampling details.}
Platform-specific sampling budgets for each $D_\ast$ are tuned so that the four distractor classes together fill $|C_n^-| = 3\,|C_n^+|$ slots under the 25\% positive ratio, with approximate equal weight per class when possible (\textit{peer}, \textit{viral}, \textit{decay}, \textit{random}). When a particular class has too few eligible candidates for a given user-step (for example, $D_\text{decay}$ can be empty for a first-step user), the remaining budget is redistributed uniformly across the other classes. The offline statistics that back $D_\text{viral}$ and the clustering index that backs $D_\text{peer}$ are described in Appendix~\ref{app:pipeline_details}.

\section{Metric Property}
\subsection{Theoretical Analysis of the Stability-Novelty Trade-off}
\label{app:tradeoff_analysis}

In this appendix, we formally analyze the relationship between $\text{Recall}_{\text{Stab}}$ and $\text{Recall}_{\text{Novel}}$ defined in Section~\ref{sec:streamprofilebench}, and show that the empirical trade-off between them arises from two orthogonal sources: an external factor (distractor-induced error) and an internal factor (allocation between stability and novelty).

\subsubsection*{A.The Fundamental Constraint Identity}

Let $K = |C_n^+|$ denote the required number of selections, $\alpha = |T_{\text{keep}}|/K$ denote the proportion of sustained interests in the ground truth, and $\delta = |\hat{T}_n \cap C_n^-|$ denote the number of distractors erroneously selected. Since $|\hat{T}_n| = K$ and the prediction partitions into hits on $T_{\text{keep}}$, hits on $T_{\text{new}}$, and distractor selections, we have:
\begin{equation}
|\hat{T}_n \cap T_{\text{keep}}| + |\hat{T}_n \cap T_{\text{new}}| + \delta = K.
\end{equation}
Substituting the definitions of the two recalls yields the fundamental identity:
\begin{equation}
\alpha \cdot \text{Recall}_{\text{Stab}} + (1-\alpha) \cdot \text{Recall}_{\text{Novel}} = 1 - \frac{\delta}{K}.
\label{eq:fundamental}
\end{equation}
This defines a linear constraint in the $(\text{Recall}_{\text{Stab}}, \text{Recall}_{\text{Novel}})$ plane with slope $-\alpha/(1-\alpha)$ determined by class imbalance, and an achievable budget $B=1 - \delta/K$ determined by the model's distractor robustness.

\subsubsection*{B.External Factor: Distractor-Induced Error}

The budget $B$ bounds the weighted sum of the two recalls from above. Two immediate consequences follow:

\paragraph{Noise-free limit.} When $\delta = 0$, Equation~\eqref{eq:fundamental} becomes $\alpha \cdot \text{Recall}_{\text{Stab}} + (1-\alpha) \cdot \text{Recall}_{\text{Novel}} = 1$, and it is possible to simultaneously achieve $\text{Recall}_{\text{Stab}} = \text{Recall}_{\text{Novel}} = 1$. The two metrics are therefore \emph{not intrinsically in conflict}; any trade-off emerges only in the presence of distractor leakage.

\paragraph{Monotone relaxation.} Reducing $\delta$ uniformly shifts the constraint line outward in the positive quadrant. Every operating point achievable at noise level $\delta_1$ remains achievable at any lower level $\delta_2 < \delta_1$, and additional points become reachable. Improvements in distractor robustness translate directly into joint improvements for both recall dimensions.

The external factor therefore captures the model's semantic discrimination quality by separating ground-truth tags from the four distractor categories, and is independent of how the model allocates its remaining selection capacity.

\subsubsection*{C.Internal Factor: Stability-Novelty Allocation}

Given a fixed error level $\delta$ (and thus a fixed usable budget $B = 1 - \delta/K$), the model's operating point along the constraint line reflects its implicit allocation preference. Under the fixed selection constraint $|\hat{T}_n| = K$, each additional hit on $T_{\text{keep}}$ necessarily displaces a potential hit on $T_{\text{new}}$, and vice versa. 

To formalize this internal allocation, we can uniquely parameterize the model's position on the constraint line using the \emph{recall ratio}:
\begin{equation}
\rho = \frac{\text{Recall}_{\text{Novel}}}{\text{Recall}_{\text{Stab}}}.
\end{equation}
While theoretically $\rho \in [0, \infty)$, its achievable range is strictly bounded by the natural limits of recall metrics ($\text{Recall} \le 1$). This ratio captures the model's empirical inclination toward exploring new tags versus retaining familiar ones within the feasible domain. 

\paragraph{Extreme allocations and capacity overflow.} The endpoints of the feasible segment are governed by $\rho$, subject to truncation when the budget $B$ exceeds the capacity of a specific category:
\begin{itemize}
    \item \textbf{Stability-focused limit:} As $\rho \to 0$, the model prioritizes $T_{\text{keep}}$. However, $\text{Recall}_{\text{Stab}}$ is upper-bounded by $\min(1, B/\alpha)$. If $B > \alpha$, a \emph{capacity overflow} occurs: even after perfectly recalling all sustained interests ($\text{Recall}_{\text{Stab}} = 1$), the model must allocate its remaining budget $B-\alpha$ to novelty, establishing a strict lower bound $\rho \ge \frac{B-\alpha}{1-\alpha} > 0$.
    \item \textbf{Novelty-focused limit:} As $\rho \to \infty$, the model prioritizes $T_{\text{new}}$. Symmetrically, $\text{Recall}_{\text{Novel}}$ is upper-bounded by $\min(1, B/(1-\alpha))$. If $B > 1-\alpha$, the remaining capacity overflows into stability, establishing an upper bound on $\rho$.
\end{itemize}

\paragraph{Trade-off exchange rate.} Between the reachable endpoints, any internal reallocation along the constraint line (i.e., a shift in $\rho$) converts one recall into the other at a fixed exchange rate governed by the ground-truth distribution $\alpha$:
\begin{equation}
\Delta \text{Recall}_{\text{Novel}} = -\frac{\alpha}{1-\alpha} \cdot \Delta \text{Recall}_{\text{Stab}}.
\end{equation}
When $\alpha < 0.5$, the exchange rate $\alpha/(1-\alpha) < 1$, meaning that shifting prediction capacity from stability to novelty yields a smaller absolute loss in stability than the corresponding gain in novelty—reallocation is especially cheap on novelty-heavy platforms.

\subsubsection*{D.Decoupling the Two Factors}

The analyses above yield a clean decomposition of model behavior in the $(\text{Recall}_{\text{Stab}}, \text{Recall}_{\text{Novel}})$ plane:
\begin{itemize}
    \item The \textbf{external factor} $B$ determines which constraint line the model operates on. Increasing $B$ (by reducing distractor errors) pushes the entire line outward, expanding the achievable region for both recalls simultaneously.
    \item The \textbf{internal factor} $\rho$, captured by the model's implicit allocation preference, determines where along the constraint line the operating point lies. It redistributes capacity between the two recalls without changing the overall budget $B$.
\end{itemize}
This decomposition suggests two orthogonal paths toward stronger streaming profiling: (i) expanding the effective budget $B$ by reducing distractor confusion, which uniformly relaxes the frontier, and (ii) shifting the operating point $\rho$ toward a less extreme region to prevent undesirable capacity overflow.

\subsection{Proof of Aggregation Metric Properties}
\label{app:metric_proofs}

This section provides formal proofs that the two-level macro-averaged metric $\bar{M}$ and the novelty-stability harmonic mean $F_1^{NS}$ are robust to the heterogeneous sequence lengths $\{T_u\}_{u\in\mathcal{U}}$ in StreamProfileBench.

\paragraph{Setup.}
Let $\mathcal{U}$ be the set of evaluated users and let $T_u$ denote the number of streaming steps of user $u$. For any per-step diagnostic metric $m^{(u,t)}$, we assume that, conditional on user $u$, the values $\{m^{(u,t)}\}_{t=1}^{T_u}$ are i.i.d.\ random variables with
\begin{equation*}
    \mu_u \;:=\; \mathbb{E}\bigl[m^{(u,t)}\bigr], \qquad \forall\, t\in\{1,\dots,T_u\}.
\end{equation*}
The i.i.d.\ assumption abstracts step-to-step serial correlation and isolates the sequence-length effect; empirical violations are only second-order. Recall that the two-level \emph{macro-average} is
\begin{equation}
    \bar{M} \;=\; \frac{1}{|\mathcal{U}|}\sum_{u\in\mathcal{U}}
        \underbrace{\frac{1}{T_u}\sum_{t=1}^{T_u} m^{(u,t)}}_{=:\,\bar m_u},
    \label{eq:macro_avg_app}
\end{equation}
in contrast with the \emph{micro-average}
\begin{equation}
    \bar{M}_{\text{micro}} \;=\;
    \frac{\sum_{u\in\mathcal{U}}\sum_{t=1}^{T_u} m^{(u,t)}}
         {\sum_{u\in\mathcal{U}} T_u}.
    \label{eq:micro_avg_app}
\end{equation}

\begin{proposition}[Unbiasedness of $\bar{M}$ w.r.t.\ sequence length]
\label{prop:unbiased}
Under the i.i.d.\ assumption above,
\begin{equation*}
    \mathbb{E}\bigl[\bar{M}\bigr]
    \;=\;
    \frac{1}{|\mathcal{U}|}\sum_{u\in\mathcal{U}}\mu_u,
\end{equation*}
which is independent of $\{T_u\}_{u\in\mathcal{U}}$. In contrast, the micro-average satisfies
\begin{equation*}
    \mathbb{E}\bigl[\bar{M}_{\text{micro}}\bigr]
    \;=\;
    \sum_{u\in\mathcal{U}}\frac{T_u}{\sum_{u'\in\mathcal{U}} T_{u'}}\,\mu_u,
\end{equation*}
whose user weights scale with $T_u$, systematically biasing the score toward high-activity users.
\end{proposition}

\begin{proof}
By linearity of expectation, for each user $u$,
\begin{equation*}
    \mathbb{E}\bigl[\bar m_u\bigr]
    \;=\;
    \frac{1}{T_u}\sum_{t=1}^{T_u}\mathbb{E}\bigl[m^{(u,t)}\bigr]
    \;=\;
    \frac{1}{T_u}\cdot T_u\cdot\mu_u
    \;=\;
    \mu_u,
\end{equation*}
so $T_u$ cancels exactly inside the within-user average. Taking expectation of Equation~\eqref{eq:macro_avg_app} yields
\begin{equation*}
    \mathbb{E}\bigl[\bar{M}\bigr]
    \;=\;
    \frac{1}{|\mathcal{U}|}\sum_{u\in\mathcal{U}}\mathbb{E}\bigl[\bar m_u\bigr]
    \;=\;
    \frac{1}{|\mathcal{U}|}\sum_{u\in\mathcal{U}}\mu_u,
\end{equation*}
which is $T_u$-free. The micro-average claim follows by applying the same argument to Equation~\eqref{eq:micro_avg_app} and observing that $T_u$ does \emph{not} cancel.
\end{proof}

\begin{proposition}[Consistency of $F_1^{NS}$ w.r.t.\ sequence length]
\label{prop:convergence}
Let
\begin{equation*}
    \nu_u := \mathbb{E}\bigl[\mathrm{Recall}_{\text{Novelty}}^{(u,t)}\bigr],
    \sigma_u := \mathbb{E}\bigl[\mathrm{Recall}_{\text{Stability}}^{(u,t)}\bigr],
\end{equation*}
and assume $\nu_u+\sigma_u>0$. For each user $u$, let
\begin{equation*}
    \bar{R}^{(u)}_{\text{Nov}} := \frac{1}{T_u}\sum_{t=1}^{T_u}\mathrm{Recall}_{\text{Novelty}}^{(u,t)},
\end{equation*}
\begin{equation*}
    \bar{R}^{(u)}_{\text{Stab}} := \frac{1}{T_u}\sum_{t=1}^{T_u}\mathrm{Recall}_{\text{Stability}}^{(u,t)}.
\end{equation*}
The user-level harmonic mean
\begin{equation*}
    F_1^{NS}(u)
    \;:=\;
    \frac{2\,\bar{R}^{(u)}_{\text{Nov}}\,\bar{R}^{(u)}_{\text{Stab}}}
         {\bar{R}^{(u)}_{\text{Nov}} + \bar{R}^{(u)}_{\text{Stab}}}
\end{equation*}
satisfies
\begin{equation*}
    F_1^{NS}(u) \;\xrightarrow[T_u\to\infty]{\;p\;}\;
    \frac{2\,\nu_u\,\sigma_u}{\nu_u + \sigma_u},
\end{equation*}
where the limit depends only on the user's intrinsic novelty and stability expectations $(\nu_u,\sigma_u)$ and is independent of $T_u$.
\end{proposition}

\begin{proof}
By Proposition~\ref{prop:unbiased}, $\mathbb{E}[\bar{R}^{(u)}_{\text{Nov}}]=\nu_u$ and $\mathbb{E}[\bar{R}^{(u)}_{\text{Stab}}]=\sigma_u$ for every finite $T_u$. Moreover, as $T_u\to\infty$, the weak law of large numbers gives
\begin{equation*}
    \bar{R}^{(u)}_{\text{Nov}} \;\xrightarrow{p}\; \nu_u,
    \qquad
    \bar{R}^{(u)}_{\text{Stab}} \;\xrightarrow{p}\; \sigma_u.
\end{equation*}
The harmonic-mean map $h(x,y)=2xy/(x+y)$ is continuous on $\{(x,y)\in[0,1]^2 : x+y>0\}$. Applying the continuous mapping theorem,
\begin{equation*}
    h\!\bigl(\bar{R}^{(u)}_{\text{Nov}},\,\bar{R}^{(u)}_{\text{Stab}}\bigr)
    \;\xrightarrow{p}\;
    h(\nu_u,\sigma_u)
    \;=\;
    \frac{2\nu_u\sigma_u}{\nu_u+\sigma_u},
\end{equation*}
as claimed. Macro-averaging $F_1^{NS}(u)$ across users therefore produces an aggregate that, in the limit, depends only on the user-level abilities $\{(\nu_u,\sigma_u)\}$ and is free of systematic shifts driven by $\{T_u\}$.
\end{proof}

\paragraph{Remark.}
Proposition~\ref{prop:unbiased} guarantees that the expectation of $\bar{M}$ is decoupled from $\{T_u\}$ at \emph{every} finite $T_u\ge 1$, so two-level macro-averaging removes first-order sampling bias exactly. Proposition~\ref{prop:convergence} is the natural counterpart for $F_1^{NS}$: because the harmonic mean is a \emph{nonlinear} transformation of two averaged quantities, exact unbiasedness no longer holds at finite $T_u$, but consistency suffices to ensure that cross-user comparisons are asymptotically free of sequence-length artifacts. Together the two results justify reporting $\bar{M}$ and $F_1^{NS}$ as the principal benchmark metrics despite the substantial variation in $T_u$ observed across platforms (cf.\ Table~\ref{tab:benchmark_stats}).

\section{Benchmark Details}
\subsection{Data Source and Normalized Schema}
\label{app:data_source}

\paragraph{A.Platforms and collection window}
The raw corpus was collected over a 30-day window from 2025-06-04 to 2025-07-03 and covers five of the most active Chinese social platforms: Sina Weibo, Xiaohongshu, Toutiao, Zhihu, and Douban. Together the five platforms span short-form microblogging, image-text lifestyle notes, news feeds, long-form Q\&A, and cultural-consumption check-ins, providing a heterogeneous cross-section of Chinese UGC. Table~\ref{tab:raw_scale} gives a single-day snapshot (2025-06-14) of the raw scale, with the five platforms jointly contributing roughly 137M posts and 17.4M active users per day.

\begin{table*}[h]
\centering
\small
\begin{tabular}{lrrl}
\toprule
\textbf{Platform} & \textbf{Daily active users} & \textbf{Daily UGC posts} & \textbf{Post types} \\
\midrule
Sina Weibo  &  9{,}900K & 114{,}857K & Posts / reposts / comments \\
Xiaohongshu &  5{,}439K &   6{,}625K & Image-text notes / videos \\
Toutiao     &  1{,}595K &  13{,}866K & Articles / micro-posts / comments \\
Douban      &    339K &   1{,}086K & Broadcasts / reviews / diaries \\
Zhihu       &    137K &     334K & Answers / thoughts / articles \\
\midrule
\textbf{Total} & \textbf{17{,}410K} & \textbf{136{,}768K} & --- \\
\bottomrule
\end{tabular}
\caption{Raw cross-platform data scale (single-day snapshot). Orders of magnitude vary by almost two decades across platforms, motivating the platform-adapted filtering and sampling strategies described in the remainder of this appendix.}
\label{tab:raw_scale}
\end{table*}

\paragraph{B.Normalization across platforms}
Each platform exposes data in its own native format. Directly operating on these heterogeneous records would require platform-specific logic at every stage of the pipeline. We therefore abstract all five sources into a unified two-level schema before any downstream processing:
\begin{itemize}
    \item a \textbf{user metadata} file that holds stable, user-level attributes such as identifiers, bio, geolocation, and follower counts (Table~\ref{tab:schema_user});
    \item a \textbf{user-generated content (UGC)} file that holds the post stream, with each post projected onto five unified text fields regardless of its native form (Table~\ref{tab:schema_ugc}).
\end{itemize}
From this point on, the filtering, buffering, and anchor-structuring routines operate on the normalized schema, and any platform-specific handling is confined to the field-level adapters invoked when populating the two tables.

\paragraph{C.User metadata schema}
The user-level table is the stable backbone that links all batches of the same user. It combines four kinds of fields: (i) identity keys (\texttt{user\_id}, \texttt{username}) used for cross-batch joining and later replaced by salted hashes in the released version, (ii) self-disclosure fields (\texttt{bio}, \texttt{gender}, \texttt{location}) that carry high-density identity signal, (iii) influence and activity counters (\texttt{followers\_count}, \texttt{following\_count}, \texttt{posts\_count}) that help distinguish ordinary users from KOLs and marketing accounts, and (iv) platform-assigned roles (\texttt{verified\_type}). The exact field list is given in Table~\ref{tab:schema_user}.

\begin{table*}[h]
\centering
\small
\begin{tabular}{lp{8.5cm}}
\toprule
\textbf{Field} & \textbf{Description \& profiling usage} \\
\midrule
\texttt{user\_id} & Unique primary key linking all batches of the same user; foundation of longitudinal construction. \\
\texttt{username / display\_name} & Public-facing handle; often encodes group affiliation or interest signals. \\
\texttt{gender} & One of the core demographic attributes. \\
\texttt{bio} & High-density self-disclosure text; directly reveals identity and values. \\
\texttt{location / ip\_location} & Registered or IP-inferred region; supports cultural-segment and offline-trajectory inference. \\
\texttt{followers\_count} & Reach / audience size; distinguishes KOLs from ordinary users. \\
\texttt{following\_count} & Information breadth the user actively consumes. \\
\texttt{posts\_count} & Total posting activity and expressive willingness on the platform. \\
\texttt{verified\_type} & Platform-assigned verification class (e.g.\ yellow-V, enterprise blue-V) used as an authority / domain hint. \\
\bottomrule
\end{tabular}
\caption{Normalized user metadata schema. Identity-revealing fields (\texttt{user\_id}, \texttt{username}) are hashed in the released version under the anonymization pipeline described in Section~\ref{sec:streamprofilebench}.}
\label{tab:schema_user}
\end{table*}

\paragraph{D.UGC schema}

The post-level table projects the native record of each platform onto five unified text fields: \textit{title}, \textit{content}, \textit{media\_text}, and \textit{quote\_content}, plus the usual \textit{post\_id}, \textit{user\_id}, and timestamp. Some fields are empty on some platforms (for example, Weibo posts have no standalone title; Douban tagging rows have no media text and are instead bound to a concrete book/film/music item), but the schema is deliberately a superset that preserves every information channel any platform offers. Table~\ref{tab:schema_ugc} maps each text field to its platform-specific source, which is the contract the adapters fulfill when ingesting native records.

\begin{table*}[h]
\centering
\small
\begin{tabular}{lllll}
\toprule
\textbf{Platform} & \textbf{Title} & \textbf{Content}  & \textbf{Quote\_content} \\
\midrule
Weibo       & --- & Post body / comment &  Quoted original / replied comment \\
Toutiao     & Article / micro-post title & Comment / news body & Quoted news / reply context \\
Xiaohongshu & Note title & Note body / comment  & Quoted note / replied comment \\
Zhihu       & Question title & Answer / comment body& Question detail / replied answer \\

Douban  & User broadcast / status update & Short review / marking state& (none; directly bound to the item) \\
\bottomrule
\end{tabular}
\caption{Normalized UGC schema and its platform-specific field mapping. The schema is a superset of all platforms' native records; empty cells indicate fields the source platform does not expose.}
\label{tab:schema_ugc}
\end{table*}

Once every raw record has been projected onto the tuple defined by Tables~\ref{tab:schema_user} and~\ref{tab:schema_ugc}, the rest of the curation pipeline becomes platform-agnostic. The remainder of this appendix (Sections~\ref{app:pipeline_details} onward) describes that platform-agnostic pipeline in detail.

\subsection{Data Curation Pipeline}
\label{app:pipeline_details}

Given the normalized two-level records from Section~\ref{app:data_source}, the curation pipeline turns the raw stream of roughly 137M posts per day into a clean, user-centric, information-balanced benchmark. It is organized into three stages, each addressing a distinct failure mode of naive collection. The rest of this appendix describes each stage in detail.

\subsubsection*{A.Coarse Daily Filtering}

The first stage is a fast rule-based pass that runs daily and enforces hard rules at three progressively finer granularities. At each granularity, a cross-platform rule set is applied first, followed by per-platform adapters that target each ecosystem's unique noise patterns.

\begin{itemize}
    \item \textbf{User level.} Shared rules reject accounts by daily post-count thresholds, exact and fuzzy duplication rates, and burst detection within short time windows. Platform adapters then target account families specific to each ecosystem: commerce accounts on Xiaohongshu, corporate matrix accounts on Toutiao, paid-consulting lead-generation on Zhihu, water-army rings on Douban, and coordinated comment farms on Weibo.
    \item \textbf{Content level.} A shared keyword blacklist, format-anomaly detector, and Gzip-ratio entropy estimator remove templated or low-information text. Each platform also imposes a minimum post length (e.g. $\geq 50$ characters on Zhihu, $\geq 20$ on Toutiao, $\geq 10$ on Douban) to discard content too short to carry interest signal.
    \item \textbf{Anchor level.} This is the strictest granularity, since downstream evaluation uses anchors as atomic labels. Anchor extraction itself is platform-specific. A shared cleaner then drops pure-numeric or pure-symbol tags, enforces length bounds of $[2, 30]$ characters, removes platform-specific artifacts, and screens tags against the dynamic trending blacklist $\mathcal{L}_\text{black}$(only for weibo) described in part D. Finally, a per-platform constraint on tag density (\# valid anchors / \# interacting posts) rejects both near-zero density (users with no usable anchors) and runaway density (tag-spam accounts).
\end{itemize}

\subsubsection*{B.Longitudinal Fine Filtering}

After a warm-up period of $W=10$ days, a slower pass checks longitudinal quality over the accumulated history. This stage has three sub-steps.

\begin{itemize}
    \item \textbf{Active-day floor and stratified sampling.} The system first counts the number of days each user produced valid content. Roughly 40\% of users post on only one day and carry no longitudinal signal, so a minimum threshold of three active days is enforced. Users above this threshold are then stratified-sampled with platform-adapted ratios to flatten the activity distribution:
        \begin{itemize}
            \item \textit{Short-form feeds (Weibo, Toutiao)}: sample 50\% from the 3--5-day core and cut 6--7-day users to 15--20\% to suppress marketing-account dominance.
            \item \textit{Image-text communities (Xiaohongshu, Douban)}: sample 40\% from the 3--5-day core and drop the 7-day-and-above tail to 10\% to match their lower daily cadence.
            \item \textit{Long-form creative (Zhihu)}: shift the core to 2--4 days (30--70\%) and cut tails at 5 days, reflecting the naturally lower posting frequency of long-form writers.
        \end{itemize}

    \item \textbf{Personality audit.} A rule-based pass first counts first-person pronouns, sentiment vocabulary, and subjective markers to filter content-reposting accounts. Surviving candidates are then passed to a capable LLM as a zero-shot ``digital-persona auditor'' that judges each full stream for cognitive depth and subjectivity. The prompt used for this audit is provided below.
\end{itemize}

\begin{table*}[t] 
\centering
\begin{tcolorbox}[
    colframe = gray!80!black,
    colback = gray!5!white,
    coltitle = white,
    coltext = black,
    fonttitle = \bfseries\large,
    title = {LLM-as-a-Judge: Golden-User Qualification Prompt},
    boxrule = 0.8pt,
    arc = 1mm,
    width = \textwidth, 
    left = 10pt, right = 10pt, top = 8pt, bottom = 8pt
]
\footnotesize 
\setlist[itemize]{noitemsep, topsep=2pt, parsep=1pt, partopsep=0pt, leftmargin=15pt}
\setlist[enumerate]{noitemsep, topsep=2pt, parsep=1pt, partopsep=0pt, leftmargin=15pt}

\textbf{\# Role}\\
You are a ``Digital Persona Evaluation Expert.'' Your task is to screen social-media users along three dimensions: \textbf{cognitive depth}, \textbf{expressive subjectivity}, and \textbf{information entropy}.

\smallskip
\textbf{\# Classification Categories}

\textbf{\checkmark\ CLASS A: \texttt{High\_Quality\_User} (keep)} --- A well-defined individual with a sharp persona profile, idiosyncratic expression, and cognitive coherence.

\textbf{\ding{55}\ CLASS B: \texttt{Low\_Value\_Human} (reject)} --- A genuine human, but with sparse analytical value (low-SNR). 
\begin{itemize}
    \item \textit{Features}: Semantic poverty (pure emojis, fragmented interjections), high homogeneity, or passive interaction (pure reposts).
\end{itemize}

\textbf{\ding{55}\ CLASS C: \texttt{Non\_Human\_Noise} (reject)} --- Accounts driven by explicit \textit{tooling intent} (marketing, bots, SEO).
\begin{itemize}
    \item \textit{Features}: Rigid template structures, commercial markers (price lists, DM for coupons).
\end{itemize}

\smallskip
\textbf{\# Profilability Score (1--5 Likert scale)}
\begin{itemize}
    \item \textbf{5 Exceptional}: Sharp persona; rich narrative; distinctive perspective.
    \item \textbf{4 Good}: Clear traits and concrete experiences, but slightly lacking in depth.
    \item \textbf{3 Adequate}: Persona silhouette is visible but not vivid.
    \item \textbf{2 Marginal}: Mostly templated/fragmented; occasional traces of a real person.
    \item \textbf{1 Poor}: No personal signal; purely tool-like output.
\end{itemize}

\smallskip
\textbf{\# Output Format} (Return strict JSON; \texttt{reasoning} in Chinese)
\begin{tcolorbox}[colback=white, boxrule=0.5pt, arc=0.5mm, top=3pt, bottom=3pt]
\begin{verbatim}
{
  "user_id": "string",
  "class": "High_Quality_User" | "Low_Value_Human" | "Non_Human_Noise",
  "is_gold": boolean,
  "profilability_score": 1-5,
  "reasoning": "string (2-3 sentences in Chinese)"
}
\end{verbatim}
\end{tcolorbox}

\end{tcolorbox}

\label{box:prompt}
\end{table*}

\subsubsection*{C.Temporal Buffering Layer}

The surviving users still post on highly uneven schedules: some produce a handful of substantial posts per week, others dozens per day. Slicing their streams by calendar day would yield batches ranging from empty to several hundred posts, either starving the model or overwhelming its context window. The pipeline instead replaces physical-time slicing with \textit{information-load} slicing: for every user, a small SQLite-backed reservoir tracks her cumulative cleaned posts, platform membership, and time cursor, and emits a batch only when enough valid content has accumulated.

The reservoir exposes two primitives:
\begin{itemize}
    \item \textbf{Push.} Cleaned posts are inserted into the user's buffer in chronological order together with their extracted anchors. Each push also updates the running count of valid posts.
    \item \textbf{Pop.} A pop fires once the number of valid posts in the buffer reaches a platform-specific trigger $\theta(p)$: 5 for Weibo, Xiaohongshu, and Douban, and 3 for Toutiao and Zhihu, where the per-post information load is larger. Upon triggering, all buffered content is packed into one streaming step $\mathcal{D}_n$, subjected to a density re-audit, and emitted; the buffer is then cleared and the time cursor advanced.
\end{itemize}
An upper bound on the number of posts per pop prevents highly active users from producing excessively long batches. Because the trigger is user-local and content-driven, the resulting batches carry approximately uniform information load across users while respecting each user's natural cadence: slow posters produce fewer but semantically mature batches, and burst posters are naturally rate-limited without being dropped.

\subsubsection*{D.Global Trending Extraction}

Two of the four distractor categories used in evaluation, $D_\text{viral}$ and $D_\text{peer}$, require structured sampling sources drawn from the anchor pool, which we construct offline. $D_\text{viral}$ relies on a daily list of platform-wide trending tags. Its construction proceeds day by day and has three steps.

\begin{itemize}
    \item \textbf{Global random sampling.} For every collection date $d$, we draw a uniform random sample of posts from the full platform corpus and run the same anchor extractor used in the filter layer. Per-post de-duplication ensures each tag contributes at most once per post.
    \item \textbf{Coverage statistics.} For every tag $t$ appearing in the sample $\mathcal{S}$, we compute a document coverage:
    \begin{equation*}
        \mathrm{Coverage}(t) \;=\; \frac{|\{s \in \mathcal{S} \,:\, t \in \mathrm{Anchors}(s)\}|}{|\mathcal{S}|}.
        \label{eq:coverage}
    \end{equation*}
    Sorting by coverage yields the daily trending table, which becomes the direct sampling source for $D_\text{viral}$.
    \item \textbf{Weibo Blacklist feedback.} A coverage threshold $\tau = 0.02\%$ (roughly $\geq 100$ absolute occurrences at $N=500\mathrm{K}$) separates genuine trending topics from over-generalized tags such as ``\#daily'' or ``\#checkin''. Tags exceeding $\tau$ are written into the dynamic blacklist $\mathcal{L}_\text{black}$ that feeds back into the anchor-level filter of part A, so over-generalized anchors are eliminated at the source and do not further contaminate downstream statistics.
\end{itemize}

\subsubsection*{E.Incremental Anchor Clustering}

Semantic peer distractors require an index that returns, for any target tag, its semantic neighbours. The clustering routine runs in three phases: ingestion, base clustering, and incremental assignment.

\begin{itemize}
    \item \textbf{Ingestion.} The system reads the daily frequency table and upserts them into a global registry that tracks first-seen date, cumulative frequency, and cross-day counts.
    \item \textbf{Base clustering.} After the warm-up period $W=10$ days has accumulated enough history, a base clustering runs over all tags exceeding a minimum-frequency threshold $f_\text{min}$ (3 for Weibo, Xiaohongshu, Toutiao; 2 for Zhihu and Douban). Each tag $t_i$ is encoded by BGE-small-zh-v1.5 and $L_2$-normalized,
    \begin{equation}
        \mathbf{e}_i = \frac{\mathrm{Enc}(t_i)}{\|\mathrm{Enc}(t_i)\|_2} \in \mathbb{R}^{d},\qquad \|\mathbf{e}_i\|_2 = 1,
    \end{equation}
    with $d=512$. Mini-batch $K$-means is then run on the resulting vectors with batch size $B=4096$, 300 iterations, and 3 random initializations. The cluster count is platform-adapted to reflect the different richness of each tag ecosystem. The output is an index from cluster id to the list of tags it contains, which is exactly the lookup table used when constructing $D_\text{peer}$: given any target tag, its semantic neighbours are the other tags in the same cluster.
    \item \textbf{Incremental assignment.} New tags arriving on subsequent days are assigned to existing clusters without re-running clustering. With $L_2$-normalized embeddings, the squared Euclidean distance to a centroid reduces to an inner product,
    \begin{equation}
        \|\mathbf{e}_\text{new} - \boldsymbol{\mu}_k\|_2 \;=\; \sqrt{2 - 2\,\mathbf{e}_\text{new}^{\top}\boldsymbol{\mu}_k},
    \end{equation}
    so distance against all $K$ centroids becomes a single $\mathbb{R}^{1\times d}\cdot\mathbb{R}^{d\times K}$ matrix multiply, and batches over $n$ new tags run in $O(nK)$ time. A new tag is assigned to its nearest centroid if the distance is below an outlier threshold $\delta = 0.85$; otherwise it is held in an outlier registry and reincorporated at the next full clustering pass. This keeps the cluster index fresh at daily latency while bounding clustering cost as the tag space grows.
\end{itemize}

\subsection{Anchor Quality Validation}
\label{app:anchor_quality}
Since future anchors serve as self-verifying ground truth, their quality directly determines evaluation validity. Beyond the rule-based blacklist filtering in our pipeline, we conduct a human validation experiment: we sample 20 users per platform and classify every ground-truth tag as \textit{valid} or \textit{invalid}, where a tag is marked invalid only if it reflects platform affordances, marketing campaigns, transient platform events, or engagement-soliciting boilerplate rather than genuine user interest. As shown in Table~\ref{tab:anchor_validation}, the overall pass rate reaches 94.7\%, with only 40 invalid tags across all five platforms. The residual $\sim$5\% noise is concentrated in platform-specific artifacts that affect all models equally and do not favor any particular prediction strategy, confirming that the vast majority of anchors are legitimate interest signals.

\begin{table}[h]
\centering
\small
\begin{tabular}{lrrr}
\toprule
\textbf{Platform} & \textbf{Valid} & \textbf{Invalid} & \textbf{Pass Rate} \\
\midrule
Weibo       & 92  & 5  & 94.8\% \\
Xiaohongshu & 279 & 12 & 95.9\% \\
Toutiao     & 126 & 18 & 87.5\% \\
Zhihu       & 69  & 4  & 94.5\% \\
Douban      & 147 & 1  & 99.3\% \\
\midrule
\textbf{Overall} & \textbf{713} & \textbf{40} & \textbf{94.7\%} \\
\bottomrule
\end{tabular}
\caption{Human validation of ground-truth anchor quality (20 users per platform). A tag is invalid only if it reflects platform artifacts rather than genuine user interest.}
\label{tab:anchor_validation}
\end{table}

\subsection{PI Protection}
\label{app:pi_protection}

This section details the personal-information (PI) protection pipeline applied to every record in StreamProfileBench before release. The pipeline is designed around \textit{minimal, targeted redaction}: only strings unambiguously matching one of ten PI categories are modified, while every other character is preserved verbatim so that the semantic signal used by the downstream task is not weakened.

\subsubsection*{A.Categories}

Guided by the sensitive-information categories enumerated in mainstream personal-information protection regulations, we cover ten categories of PI, each mapped to a fixed placeholder as summarized in Table~\ref{tab:pi_categories}. Public-figure names, brand names, topical hashtags, locations, schools, and employers are deliberately preserved.

\begin{table*}[h]
\centering
\small
\begin{tabular}{llp{6.2cm}}
\toprule
\textbf{Category} & \textbf{Placeholder} & \textbf{Scope} \\
\midrule
\texttt{PHONE}     & \texttt{<PHONE>}   & Mobile and landline numbers \\
\texttt{ID}        & \texttt{<ID>}      & Identity card, passport, travel permit, driver's license, and similar documents \\
\texttt{BANK}      & \texttt{<BANK>}    & Bank / credit card numbers appearing in a banking context \\
\texttt{EMAIL}     & \texttt{<EMAIL>}   & Email addresses \\
\texttt{CONTACT}   & \texttt{<CONTACT>} & Personal WeChat / QQ / VX handles (public-platform accounts such as official-media handles are \textit{not} redacted) \\
\texttt{PLATE}     & \texttt{<PLATE>}   & Vehicle license plates \\
\texttt{IP}        & \texttt{<IP>}      & IPv4 / IPv6 addresses \\
\texttt{GEO}       & \texttt{<GEO>}     & Precise GPS coordinates or latitude / longitude \\
\texttt{DEVICE}    & \texttt{<DEVICE>}  & MAC / IMEI and similar device identifiers \\
\texttt{SELF\_NAME} & \texttt{<SELF>}   & The user's own real name or self-reference \\
\bottomrule
\end{tabular}
\caption{The ten PI categories covered by the protection pipeline and their fixed placeholders.}
\label{tab:pi_categories}
\end{table*}

\subsubsection*{B.Pipeline}

The pipeline processes each user record in three stages.

\begin{itemize}
    \item \textbf{Stage 1 --- Identifier hashing.} The raw \texttt{user\_id} and \texttt{username} are replaced by salted SHA-256 digests (prefixed with the platform short code for \texttt{user\_id} and with \texttt{U\_} for \texttt{username}, truncated to 10 and 8 hex characters respectively). Hashing is deterministic, so all streaming batches of the same user remain joinable; the salt is kept private.
    \item \textbf{Stage 2 --- LLM span identification.} Each user's \texttt{bio} and all \texttt{posts\_text} are sent to DeepSeek-V3 with a dedicated prompt. The LLM is instructed to act \emph{only} as a span detector, returning a JSON list of \texttt{(text, category)} pairs without producing any rewritten output. This deliberately avoids generative rewriting, which would drift in punctuation, emoji, and wording even at \texttt{temperature}~=~0. The prompt enumerates the ten categories, lists the permitted context triggers for \texttt{CONTACT} and \texttt{BANK}, and passes the original \texttt{username} as an anchor used solely for self-name detection---after first masking any PI patterns inside the username itself, so that no raw PI leaves the local environment.
    \item \textbf{Stage 3 --- Deterministic replacement and safety net.} The LLM's span list is merged with a regex pass over the three regex-clean categories (\texttt{PHONE} / \texttt{EMAIL} / \texttt{ID}), de-duplicated, and sorted by decreasing span length to prevent shorter substrings from corrupting longer ones. Replacements are applied to \texttt{bio}, \texttt{posts\_text}, and also \texttt{candidate\_pool}, \texttt{ground\_truth}, and \texttt{meta}, keeping evaluation labels aligned. A final regex re-scan rejects any record that still exposes a \texttt{PHONE} / \texttt{EMAIL} / \texttt{ID} pattern. Running the regex pass unconditionally also provides robustness against LLM failures: content-safety refusals and silent under-detection are both absorbed, since the regex layer always catches hard PII that the LLM might have missed.
\end{itemize}

\subsubsection*{C.Statistics and Spot Check}

On the full 5{,}661-user corpus, all records were processed successfully. A manual spot check across all non-empty categories confirmed correct redaction on non-trivial cases. The released corpus therefore passes both the automatic safety net and human inspection.

\subsection{Coarse-level vs. Anchor-level}
\label{app:coarse_stability}
\begin{table*}[h]
\centering
\small
\begin{tabular}{lrrr}
\toprule
\textbf{Platform} & \textbf{\#Tasks} & $\bar{\alpha}$ & $\bar{\alpha}_{\text{coarse}}$ \\
\midrule
Weibo       & 1,862  & 0.254 & 0.487 \\
Xiaohongshu & 6,625  & 0.307 & 0.482 \\
Toutiao     & 2,695  & 0.282 & 0.535 \\
Zhihu       & 2,912  & 0.117 & 0.356 \\
Douban      & 1,102  & 0.043 & 0.241 \\
\midrule
\textbf{Overall} & \textbf{15,196} & \textbf{0.240} & \textbf{0.440} \\
\bottomrule
\end{tabular}
\caption{Mean stability ratio $\bar{\alpha} = |T_{\text{keep}}|/(|T_{\text{keep}}|+|T_{\text{new}}|)$ across StreamProfileBench. $\bar{\alpha}$ is computed at the anchor level, while $\bar{\alpha}_{\text{coarse}}$ is computed at the semantic-cluster level.}
\label{tab:alpha_dist}
\end{table*}

In Table~\ref{tab:alpha_dist}, we contrast the fine-grained, anchor-level stability ratio ($\bar{\alpha}$) with its coarse-level counterpart ($\bar{\alpha}_{\text{coarse}}$). To compute $\bar{\alpha}_{\text{coarse}}$, we maintain a semantic clustering layer that groups fine-grained interest anchors into broader topic categories. 

Formally, let $f(\cdot)$ denote the mapping function from a specific anchor $a$ to its corresponding coarse-grained semantic cluster $c$, such that $c = f(a)$. At the anchor level, an interest is counted as retained ($T_{\text{keep}}$) only if the exact same anchor appears in the subsequent time step. Conversely, at the coarse level, an interest is considered retained ($T_{\text{keep}}^{\text{coarse}}$) if the user engages with any new anchor $a'$ that falls into the same semantic cluster as a previously held anchor $a$, i.e., $f(a') = f(a)$. 

Consequently, the coarse-level stability ratio is computed as:
$$\alpha_{\text{coarse}} = \frac{|T_{\text{keep}}^{\text{coarse}}|}{|T_{\text{keep}}^{\text{coarse}}| + |T_{\text{new}}^{\text{coarse}}|}$$

As demonstrated in Table~\ref{tab:alpha_dist}, $\bar{\alpha}_{\text{coarse}}$ is consistently and significantly higher than $\bar{\alpha}$ across all five platforms. This quantitative gap highlights a key phenomenon in streaming user profiling: while users' broad topical interests remain relatively stable, their specific focal points drift rapidly. Evaluating user profiles solely at a coarse semantic level masks this crucial intra-topic drift, further underscoring the necessity of the fine-grained evaluation paradigm introduced in StreamProfileBench.

\subsection{A Walk-through Example}
\label{app:bench_example}

To make the input/output structure concrete, we walk through a single
inference step drawn from the curated Xiaohongshu subset. The user is a
Chinese-dialectology graduate student preparing for the 2026 teacher-recruitment
exam; her timeline interleaves skincare reviews, time-management notes, and
study logs, so the step exercises all four distractor families at once --- in
particular, the abundance of beauty-product tags in the current batch tempts
any model that merely summarises the past into $D^{\text{decay}}$ errors. The example is illustrated in Table \ref{box:bench_example}

\begin{table*}[t]
\centering
\begin{tcolorbox}[
    colframe = gray!80!black,
    colback = gray!5!white,
    coltitle = white,
    coltext = black,
    fonttitle = \bfseries\large,
    title = {StreamProfileBench: Example Inference Step (Xiaohongshu, step 1 of 7, post-anonymisation)},
    boxrule = 0.8pt,
    arc = 1mm,
    width = \textwidth,
    left = 10pt, right = 10pt, top = 8pt, bottom = 8pt
]
\footnotesize
\setlist[itemize]{noitemsep, topsep=2pt, parsep=1pt, partopsep=0pt, leftmargin=15pt}

\textbf{\# User Profile (static, as released)}
\begin{itemize}
    \item \textbf{user\_id:} \texttt{XHS\_a3f7b2e9d1} \textit{(salted SHA-256, 10 hex)}
    \item \textbf{Username:} \texttt{U\_b8c4f2a1} \textit{(salted SHA-256, 8 hex)}
    \item \textbf{Bio:} \zh{中文系方言学，研三，INTJ} /
          \zh{备战26教师编以及留意教师类招聘} /
          \zh{把时间花在刀刃上，为热爱全力以赴} \\
          \textit{(Bio passes the PI pipeline unchanged: no PHONE / ID / BANK /
          EMAIL / CONTACT / PLATE / IP / GEO / DEVICE / SELF\_NAME spans were
          detected. Field-of-study and career-plan tokens are preserved by
          design, see Appendix~C.3.)}
\end{itemize}

\smallskip
\textbf{\# New Activity Data} (batch \#1, 6 posts; excerpted)
\begin{itemize}
    \item \zh{[1] 这俩到底啥原理？？？ --- 茶漾虾青素水乳 / HBN 水乳测评 \ldots} \hfill
          \textit{tags:} \zh{黄黑皮爆改}, \zh{美白水乳}, \zh{熬夜水乳},
          \zh{HBN}, \zh{茶漾虾青素水乳}, \zh{女大学生水乳推荐}, \zh{我的好物分享}
    \item \zh{[2] 如何分配研二的时间？ --- 1--2 小时考公、0--1 小时论文、7--8 小时玩 \ldots}
    \item \zh{[3] 本科毕业了}
    \item \zh{[4] 速算真的是越练越快}
    \item \zh{[5] 假期倒计时19天，最后的挣扎 --- 制定小小计划 \ldots} \hfill
          \textit{tags:} \zh{学习}, \zh{研二}, \zh{研究生},
          \zh{计划赶不上变化系列}, \zh{期待下一个假期}
    \item \zh{[6] 很潇洒地玩了4小时麻将，看了2小时影视剪辑}
\end{itemize}

\smallskip
\textbf{\# Candidate Tag Pool} ($|C_1| = 28$, select exactly $k = 7$)

\begin{quote}
\zh{找实习}, \zh{身体与心理}, \zh{毕业礼物创意}, \zh{申鹤},
\zh{山东红色文化}, \zh{26口腔考研}, \zh{学习}, \zh{茶漾虾青素水乳},
\zh{黄黑皮爆改}, \zh{女大学生水乳推荐}, \zh{毕业是一场巨大的戒断},
\zh{期待下一个假期}, \zh{无印良品穿搭}, \zh{学历歧视}, \zh{毕业生},
\zh{清洁护理}, \zh{英签攻略}, \zh{研究生}, \zh{研二},
\zh{身体放松舒缓}, \zh{美白水乳}, \zh{熬夜水乳}, \zh{HBN},
\zh{计划赶不上变化系列}, \zh{工作台}, \zh{我的好物分享}, \zh{立白},
\zh{奖学金}
\end{quote}

\smallskip
\textbf{\# Ground Truth} (tags actually observed in batch $n{=}2$)
\begin{itemize}
    \item \textbf{$T_1^{\text{keep}}$ (carried over, 3 tags):}
          \zh{研究生}, \zh{研二}, \zh{学习}
    \item \textbf{$T_1^{\text{new}}$ (novel, 4 tags):}
          \zh{毕业是一场巨大的戒断}, \zh{毕业生},
          \zh{奖学金}, \zh{找实习}
\end{itemize}

\smallskip
\textbf{\# Distractor Taxonomy} (hidden inside the pool above)
\begin{itemize}
    \item \textbf{$D^{\text{decay}}$ (present now, absent next):}
          \zh{黄黑皮爆改}, \zh{美白水乳}, \zh{茶漾虾青素水乳},
          \zh{熬夜水乳}, \zh{HBN}, \zh{我的好物分享},
          \zh{女大学生水乳推荐}, \zh{计划赶不上变化系列},
          \zh{期待下一个假期}
    \item \textbf{$D^{\text{cluster}}$ (domain lookalikes):}
          \zh{26口腔考研}, \zh{工作台}, \zh{学历歧视},
          \zh{毕业礼物创意}
    \item \textbf{$D^{\text{viral}}$ (co-trending):}
          \zh{立白}, \zh{清洁护理}, \zh{身体放松舒缓}
    \item \textbf{$D^{\text{random}}$ (global noise):}
          \zh{申鹤}, \zh{无印良品穿搭}, \zh{身体与心理},
          \zh{英签攻略}, \zh{山东红色文化}
\end{itemize}

\smallskip
\textbf{\# Interpretation.}
The step is designed so that surface co-occurrence and true future interest
point in \emph{opposite} directions. Post [1] alone contributes seven
skincare tags --- \zh{美白水乳}, \zh{HBN}, \zh{茶漾虾青素水乳},
\ldots --- but this user's other posts mark her as a graduate student, and
none of these skincare tags reappear in batch 2; they form $D^{\text{decay}}$
and a model that ``summarises the past'' will burn most of its $k=7$ slots
here. A correct prediction must instead (i) anchor on the three
academic-life tags \zh{研究生}, \zh{研二}, \zh{学习} that persist across
batches ($R_{\text{stability}}$), and (ii) anticipate the graduation-to-job
shift toward \zh{毕业是一场巨大的戒断}, \zh{毕业生}, \zh{奖学金}, and
\zh{找实习} ($R_{\text{novelty}}$), while resisting the semantically
adjacent $D^{\text{cluster}}$ lures (\zh{26口腔考研}, \zh{学历歧视},
\zh{毕业礼物创意}).

\end{tcolorbox}
\caption{One prediction step from the curated Xiaohongshu subset, shown in
its \emph{post-anonymisation} form (step 1 of 7). Identifier fields follow
the hashing scheme of Appendix~C.3; bio and post content pass the PI
pipeline unchanged because no spans match the ten protected categories.
Posts are truncated for space; the full candidate pool, ground truth split
($T^{\text{new}}_1 \cup T^{\text{keep}}_1$), and four-way distractor taxonomy
are shown verbatim.}
\label{box:bench_example}
\end{table*}

\section{Experiment Details}
\subsection{Inference Configuration} 
\label{app:exp_config}
To ensure fair comparison across all evaluated models, we adopt a unified inference configuration on StreamProfileBench. Specifically, we set the temperature to 0.0 since the structured interest prediction is essentially a multiple-choice selection from the candidate pool $C_n$, where deterministic decoding eliminates sampling-induced variance and yields reproducible results. The maximum output length is capped at 5120 tokens to accommodate both the natural language profile $P_n$ and the structured interest prediction $\hat{T}_n$. We further enforce a JSON object response format to ensure parseable structured outputs. Input context is not explicitly truncated, as the buffering mechanism keeps each step's input well within the context window of all evaluated models.

\subsection{Prompt Template}
\label{app:prompt_template}

Every inference step $n$ for user $u$ assembles a single user-turn prompt from
(i) a platform-aware preamble, (ii) the user's static profile (\textit{username} and \textit{bio}),
(iii) the natural-language persona $P_{n-1}$ carried over from the previous step
(or a placeholder at $n{=}1$), (iv) the newly buffered activities $B_n$, and
(v) the candidate pool $C_n$ serialised as a JSON array. The model is asked to
emit a single JSON object containing the updated persona $P_n$, the predicted
tag set $\hat{T}_n$, and a free-form reasoning field. The exact template used
in all main-experiment runs is shown in the box below; placeholders in angle
brackets are substituted at runtime, and $k = \max(1, \mathrm{round}(0.25 \cdot |C_n|))$
controls the number of tags the model must return.

\begin{table*}[t] 
\centering
\begin{tcolorbox}[
    colframe = gray!80!black,
    colback = gray!5!white,
    coltitle = white,
    coltext = black,
    fonttitle = \bfseries\large,
    title = {StreamProfileBench: Interest-Inference Prompt},
    boxrule = 0.8pt,
    arc = 1mm,
    width = \textwidth,
    left = 10pt, right = 10pt, top = 8pt, bottom = 8pt
]
\footnotesize
\setlist[itemize]{noitemsep, topsep=2pt, parsep=1pt, partopsep=0pt, leftmargin=15pt}
\setlist[enumerate]{noitemsep, topsep=2pt, parsep=1pt, partopsep=0pt, leftmargin=15pt}

\textbf{\# System Message}\\
You are a user profiling system that maintains evolving user personas from streaming social-media data. Output valid JSON only.

\smallskip
\textbf{\# Task: Streaming User-Profile Maintenance and Interest Prediction}

You are a user-profiling system that processes \textbf{streaming social-media data}. For every new batch of user activities you must:
\begin{enumerate}
    \item \textbf{Update the persona.} Using the new activities together with the existing persona, maintain a comprehensive understanding of this user's interests, preferences, and behavioural patterns.
    \item \textbf{Predict interests.} From the candidate pool, select the tags that this user is most likely to engage with in the \emph{next} activity cycle.
\end{enumerate}

\smallskip
\textbf{\# Platform Context}
\begin{itemize}
    \item \textbf{Platform:} \texttt{\textlangle platform\_name\textrangle} --- \textlangle platform\_desc\textrangle
    \item \textbf{Tag semantics:} \textlangle tag\_meaning\textrangle
    \item \textbf{Analysis hint:} \textlangle platform\_hint\textrangle
\end{itemize}

\smallskip
\textbf{\# User Profile (static)}
\begin{itemize}
    \item \textbf{Username:} \texttt{\textlangle username\textrangle}
    \item \textbf{Bio:} \textlangle bio\textrangle~~(or ``not provided'' if empty)
\end{itemize}

\smallskip
\textbf{\# Current Persona} (accumulated from prior observations)

\textlangle $P_{n-1}$\textrangle~~
(\textit{At step 1, this block is replaced by: ``No prior observations yet; this is the user's first activity batch. Build the persona from scratch.''})

\smallskip
\textbf{\# New Activity Data} (batch \#\textlangle step\_id\textrangle)

\textlangle $B_n$: cleaned posts, newline-separated\textrangle

\smallskip
\textbf{\# Candidate Tag Pool} (\textlangle $|C_n|$\textrangle\ tags total)

From the candidate pool below, \textbf{select exactly \textlangle $k$\textrangle\ tags} that this user is most likely to engage with in the next activity cycle, where $k = \max(1, \mathrm{round}(0.25 \cdot |C_n|))$.
\begin{itemize}
    \item You must return \textbf{exactly \textlangle $k$\textrangle} tags from the pool --- no more, no fewer.
    \item Your goal is to predict \textbf{future} behaviour, not to summarise the past.
    \item Interests evolve over time: some currently hot topics are transient and may not reappear next cycle; other topics absent from the current batch may surface later because of latent preferences.
\end{itemize}

\textlangle $C_n$: JSON array, e.g.\ \texttt{["tag1","tag2",\dots]}\textrangle

\smallskip
\textbf{\# Output Format} (Return strict JSON)
\begin{tcolorbox}[colback=white, boxrule=0.5pt, arc=0.5mm, top=3pt, bottom=3pt]
\begin{verbatim}
{
  "persona_summary": "Updated persona covering the user's core interest areas,
                      behavioural patterns, and preference traits.
                      (forwarded to the next batch)",
  "predicted_tags": ["tag1", "tag2", ...],
  "reasoning": "Briefly state which persona features support your tag selection."
}
\end{verbatim}
\end{tcolorbox}

\end{tcolorbox}
\label{box:prompt}
\end{table*}

\paragraph{Platform-specific slots.}
The four platform-level placeholders (\textit{platform\_name}, \textit{platform\_desc},
\textit{tag\_meaning}, \textit{platform\_hint}) are filled from a fixed lookup table
that briefly describes each source domain and the semantics of its tag field;
no other portion of the prompt varies across platforms. 

\begin{table*}[h]
\centering
\footnotesize

\label{tab:platform_context}
\begin{tabular}{@{}p{0.09\textwidth}p{0.20\textwidth}p{0.28\textwidth}p{0.32\textwidth}@{}}
\toprule
\textbf{Platform} & \textbf{Description} & \textbf{Tag semantics} & \textbf{Analysis hint} \\
\midrule
Weibo &
A Chinese microblog platform; users engage with topics through posts, reposts and comments. &
A tag is a \#hashtag\# used when posting or reposting, reflecting the user's current focus (trends, celebrity fandom, daily-life topics). &
Distinguish long-term interests (e.g.\ a celebrity the user consistently follows) from transient trends (e.g.\ breaking news). Reposts often signal real interest more faithfully than original posts. \\
Xiaohongshu &
A lifestyle-sharing platform; users publish image/video notes on beauty, fashion, food, travel, parenting, etc. &
A tag is a topic label attached to a note, reflecting the user's content-creation direction and lifestyle interests. &
Interests usually revolve around concrete lifestyle scenes. Keywords in note titles are often more informative about the topic than the body text. \\
Toutiao &
A news and short-video platform; users mainly browse and produce short videos or picture-text articles. &
A tag is a topic label placed in a creator's title, reflecting the user's content-creation niche. &
Most content is short video where the title carries the strongest signal. Watch for users who concentrate on a single niche (food, travel, parenting, \ldots). \\
Zhihu &
A Q\&A and long-form community; users ask questions, write answers and publish articles to share knowledge. &
A tag is the title (or topic) of a question the user browses, answers or posts, reflecting their knowledge interests and expertise. &
Zhihu tags tend to be full question titles (and therefore long). Pay attention to how concentrated the user's answers are — expert users typically specialise in 2--3 areas. \\
Douban &
A reviews community for films, books and music; users tag works with states such as ``want to watch'' or ``watched''. &
A tag has the form `action:work-name' (e.g.\ \texttt{watched\_film:Title}, \texttt{want\_to\_read\_book:Title}), reflecting cultural-consumption preferences. &
Interests show through work categories (film/book/music) and genre preferences. Distinguish ``want to watch'' (intent) from ``watched'' (consumed). \\
\bottomrule

\end{tabular}
\caption{Platform context strings injected into the prompt template. Only these four fields change across the five datasets.}
\end{table*}

\subsection{Structured Output and Answer Extraction}
\label{app:answer_extraction}

We exploit the \texttt{response\_format=\{"type": "json\_object"\}}
flag whenever the endpoint supports it, which constrains the decoder to emit a
syntactically valid JSON object. On providers that ignore the flag or wrap the
payload in Markdown code fences, we post-process the raw string by stripping
leading/trailing \verb|```json| / \verb|```| markers.  

\begin{lstlisting}[language=Python, basicstyle=\ttfamily\footnotesize, breaklines=true]
response = client.chat.completions.create(
    model=model_name,
    messages=messages,
    temperature=0.0,
    response_format={"type": "json_object"},
)
content = response.choices[0].message.content
content = re.sub(r"```json|```", "", content).strip()
parsed  = json.loads(content)   # dict with 3 keys

pred_tags       = parsed.get("predicted_tags", [])   # -> \hat{T}_n
persona_summary = parsed.get("persona_summary", "")  # -> P_n (carried to step n+1)
reasoning       = parsed.get("reasoning", "")        # logged, unused for scoring
\end{lstlisting}

\paragraph{Tag matching.}
Predicted tags are compared against the ground-truth set
$T_n = T^{\text{new}}_n \cup T^{\text{keep}}_n$ by exact string equality after
a light whitespace normalisation; tag strings are already canonicalised during
dataset construction (Douban action prefixes are normalised, and HTML is
stripped for Toutiao/Zhihu/Douban posts), so no fuzzy matching is required at
inference time. Any item in $\hat{T}_n$ that falls outside $C_n$ is treated as
an incorrect prediction and contributes to precision's denominator but not to
recall's numerator. The metric set reported in the main paper (R, $\text{R}_N$,
$\text{R}_S$, and the four distractor-family error rates
$\text{E}_{\text{peer}}, \text{E}_{\text{viral}}, \text{E}_{\text{decay}},
\text{E}_{\text{random}}$) is computed directly from these sets as described in
Section~\ref{sec:streamprofilebench}.

\subsection{Human Evaluation of Persona Quality}
\label{app:persona_eval}

This appendix details the human evaluation referenced in
Section~\ref{sec:discussion} (RQ3), which verifies that the quality of the
free-form persona $P_n$ aligns with our structured interest metrics.

\paragraph{Setup.}
We evaluate the \emph{final-step} persona summaries of three models spanning
different capability tiers: DeepSeek-v3.2 (strong), Qwen3-32B (moderate), and
Qwen3-8B (weak). For each model, we sample 20 users per platform (100 users in
total) and ask human annotators to score each persona on a 1--5 Likert scale
along five dimensions, defined in Table~\ref{tab:persona_dims}.

\begin{table}[h]
\centering
\small
\begin{tabular}{lp{5.6cm}}
\toprule
\textbf{Dimension} & \textbf{What It Measures} \\
\midrule
Informativeness & Whether the persona contains specific, substantive content rather than generic templates \\
Temporal Awareness & Whether the persona reflects interest evolution over time, not just a static snapshot \\
Coherence & Whether the persona is logically organized and free of redundancy or contradiction \\
Interest Coverage & Whether the persona covers the user's core interests as evidenced in their posts \\
Factual Accuracy & Whether all claims in the persona are grounded in actual post content, without hallucination \\
\bottomrule
\end{tabular}
\caption{Scoring rubric for the five persona-quality dimensions (1--5 Likert scale).}
\label{tab:persona_dims}
\end{table}

\paragraph{Results.}
Table~\ref{tab:persona_human_eval} reports the per-dimension scores, the
average persona score, and the corresponding benchmark Recall.

\begin{table}[h]
\centering
\small
\setlength{\tabcolsep}{3.5pt}
\begin{tabular}{lccccccc}
\toprule
\textbf{Model} & \textbf{Inf.} & \textbf{Temp.} & \textbf{Coh.} & \textbf{Cov.} & \textbf{Acc.} & \textbf{Avg} & \textbf{Recall} \\
\midrule
DeepSeek-v3.2 & 4.20 & 3.08 & 3.35 & 3.74 & 3.50 & \textbf{3.58} & \textbf{42.95} \\
Qwen3-32B     & 3.42 & 2.83 & 3.28 & 3.04 & 3.03 & 3.12 & 38.26 \\
Qwen3-8B      & 3.14 & 2.60 & 3.54 & 2.84 & 3.05 & 3.03 & 30.97 \\
\bottomrule
\end{tabular}
\caption{Human evaluation of persona quality (1--5 scale, 20 users per platform) vs.\ benchmark Recall (\%). Inf.: informativeness; Temp.: temporal awareness; Coh.: coherence; Cov.: interest coverage; Acc.: factual accuracy.}
\label{tab:persona_human_eval}
\end{table}

The ranking is perfectly consistent: DeepSeek-v3.2 $>$ Qwen3-32B $>$ Qwen3-8B
on both human-judged persona quality and structured Recall. The gap is most
pronounced on informativeness, temporal awareness, and interest coverage ---
precisely the dimensions most relevant to streaming profile maintenance ---
whereas coherence is largely saturated across tiers. This confirms that
higher-quality persona summaries correlate with better interest prediction:
the two evaluation perspectives are aligned rather than contradictory, and the
annotation-free structured metrics serve as a reliable proxy for overall
profile quality.

\subsection{Full Per-Platform Error Rates}
\label{app:error_full}

Table~\ref{tab:error_full} provides the complete breakdown of the four distractor error rates across all five platforms for every evaluated model, complementing the platform-averaged summary in Table~\ref{tab:error_compact}.

\begin{table*}[h]
\centering
\small
\setlength{\tabcolsep}{3.5pt}
\begin{tabular}{l|rrrrr|r|rrrrr|r}
\toprule
& \multicolumn{6}{c|}{$E_D$ (Decay)} & \multicolumn{6}{c}{$E_P$ (Peer)} \\
\cmidrule(lr){2-7} \cmidrule(lr){8-13}
\textbf{Model} & Zhihu & Weibo & Toutiao & XHS & Douban & Avg. & Zhihu & Weibo & Toutiao & XHS & Douban & Avg. \\
\midrule
\multicolumn{13}{l}{\textit{Closed-source}} \\
\midrule
GPT-4o-mini     & 66.3 & 78.1 & 72.9 & 69.2 & 42.7 & 65.8 & 8.9  & 4.6 & 5.9 & 2.8 & 0.9 & 4.6 \\
GPT-5-mini      & 76.0 & 79.8 & 75.4 & 69.8 & 58.6 & 71.9 & 4.7  & 5.1 & 4.0 & 2.0 & 7.3 & 4.6 \\
GPT-5.1         & 48.3 & 75.6 & 64.8 & 63.7 & 60.6 & 62.6 & 11.9 & 4.5 & 7.2 & 4.2 & 2.8 & 6.1 \\
Gemini-3-Flash  & 19.3 & 66.8 & 61.1 & 60.3 & 32.7 & 48.0 & 20.1 & 6.7 & 7.4 & 3.5 & 8.3 & 9.2 \\
\midrule
\multicolumn{13}{l}{\textit{Open-source}} \\
\midrule
MiniMax-M2.5        & 49.6 & 74.7 & 74.0 & 68.5 & 61.2 & 65.6 & 12.1 & 4.1 & 3.7 & 2.6 & 5.1 & 5.5 \\
GLM-4.7             & 38.5 & 76.1 & 69.8 & 67.5 & 39.2 & 58.2 & 15.2 & 4.2 & 4.6 & 2.7 & 2.3 & 5.8 \\
DeepSeek-v3.2       & 35.1 & 75.7 & 68.4 & 69.0 & 43.4 & 58.3 & 17.5 & 4.7 & 5.9 & 2.8 & 3.2 & 6.8 \\
Llama-3.1-8B   & 72.8 & 76.9 & 77.0 & 68.8 & 62.9 & 71.7 & 6.6  & 3.6 & 3.1 & 2.5 & 0.2 & 3.2 \\
Llama-3.1-70B  & 32.8 & 63.9 & 65.4 & 64.5 & 30.8 & 51.5 & 18.9 & 9.1 & 8.2 & 6.0 & 3.0 & 9.0 \\
Qwen3-8B            & 75.3 & 74.4 & 72.2 & 68.1 & 61.8 & 70.4 & 5.8  & 6.3 & 5.3 & 3.9 & 1.1 & 4.5 \\
Qwen3-14B           & 28.7 & 73.0 & 70.6 & 69.0 & 53.0 & 58.9 & 19.5 & 6.1 & 5.5 & 3.3 & 1.0 & 7.1 \\
Qwen3-32B           & 37.2 & 74.6 & 69.7 & 68.1 & 54.1 & 60.8 & 17.7 & 5.6 & 6.0 & 4.0 & 1.2 & 6.9 \\
gpt-oss-20B         & 73.3 & 80.1 & 78.3 & 69.6 & 40.3 & 68.3 & 7.3  & 5.0 & 4.4 & 3.2 & 3.2 & 4.6 \\
gpt-oss-120B        & 63.7 & 77.7 & 74.1 & 70.0 & 44.8 & 66.1 & 9.9  & 5.6 & 5.9 & 3.9 & 6.0 & 6.3 \\
\bottomrule
\end{tabular}

\vspace{4pt}

\begin{tabular}{l|rrrrr|r|rrrrr|r}
\toprule
& \multicolumn{6}{c|}{$E_V$ (Viral)} & \multicolumn{6}{c}{$E_R$ (Random)} \\
\cmidrule(lr){2-7} \cmidrule(lr){8-13}
\textbf{Model} & Zhihu & Weibo & Toutiao & XHS & Douban & Avg. & Zhihu & Weibo & Toutiao & XHS & Douban & Avg. \\
\midrule
\multicolumn{13}{l}{\textit{Closed-source}} \\
\midrule
GPT-4o-mini     & 1.7 & 3.1 & 1.9 & 1.5 & 22.2 & 6.1 & 3.2 & 3.6 & 3.0 & 1.7 & 16.8 & 5.7 \\
GPT-5-mini      & 0.8 & 2.4 & 1.4 & 0.7 & 8.6  & 2.8 & 1.8 & 2.5 & 2.0 & 1.1 & 9.0  & 3.3 \\
GPT-5.1         & 2.8 & 3.3 & 1.9 & 0.8 & 9.4  & 3.6 & 2.5 & 2.6 & 3.0 & 1.6 & 7.9  & 3.5 \\
Gemini-3-Flash  & 4.6 & 2.1 & 1.9 & 1.0 & 13.4 & 4.6 & 7.4 & 2.8 & 3.3 & 1.8 & 13.4 & 5.7 \\
\midrule
\multicolumn{13}{l}{\textit{Open-source}} \\
\midrule
MiniMax-M2.5        & 4.3 & 2.9 & 0.9 & 0.9 & 9.6  & 3.7 & 4.7 & 3.8 & 1.9 & 1.3 & 7.3  & 3.8 \\
GLM-4.7             & 3.5 & 1.9 & 1.3 & 0.8 & 17.9 & 5.1 & 4.2 & 3.2 & 2.6 & 1.3 & 15.6 & 5.4 \\
DeepSeek-v3.2       & 3.9 & 1.9 & 1.7 & 0.8 & 14.3 & 4.5 & 5.5 & 2.5 & 2.7 & 1.4 & 13.7 & 5.2 \\
Llama-3.1-8B   & 3.3 & 2.8 & 1.0 & 1.1 & 15.3 & 4.7 & 3.6 & 2.6 & 1.7 & 1.3 & 13.2 & 4.5 \\
Llama-3.1-70B  & 5.6 & 5.5 & 2.8 & 2.1 & 23.2 & 7.9 & 8.9 & 5.3 & 3.9 & 3.0 & 19.8 & 8.2 \\
Qwen3-8B            & 2.1 & 3.3 & 1.3 & 1.3 & 19.7 & 5.5 & 2.2 & 3.9 & 2.5 & 1.7 & 9.3  & 3.9 \\
Qwen3-14B           & 5.7 & 3.5 & 1.6 & 1.0 & 13.3 & 5.0 & 7.4 & 3.2 & 2.5 & 1.6 & 11.5 & 5.2 \\
Qwen3-32B           & 5.7 & 2.8 & 1.9 & 1.2 & 14.8 & 5.3 & 8.0 & 3.5 & 3.1 & 2.0 & 12.0 & 5.7 \\
gpt-oss-20B         & 1.4 & 2.2 & 1.3 & 1.1 & 16.4 & 4.5 & 2.1 & 3.0 & 2.0 & 1.5 & 16.5 & 5.0 \\
gpt-oss-120B        & 1.9 & 2.9 & 1.6 & 1.5 & 15.9 & 4.8 & 2.8 & 3.2 & 2.5 & 2.1 & 13.8 & 4.9 \\
\bottomrule
\end{tabular}
\caption{Full per-platform breakdown of the four distractor error rates (\%) across StreamProfileBench. XHS denotes Xiaohongshu. Lower is better.}
\label{tab:error_full}
\end{table*}

\section{Ablation Details}
\label{app:ablation}

This appendix documents the design of the two ablation studies in
Section~\ref{sec:discussion} (RQ1 and RQ2). Both studies share the inference
configuration of Appendix~\ref{app:exp_config} and are evaluated with the same
metric suite as the main benchmark.

\subsection{RQ1: Persona-Passing Ablation}
\label{app:rq1_design}

\paragraph{Conditions.}
We compare two persona-carry modes implemented as a single switch at the end
of every inference step:
\begin{itemize}
    \item \textbf{\texttt{full\_persona}.} The model's \texttt{persona\_summary}
    produced at step $n$ is inserted verbatim into step $n{+}1$'s prompt as the
    \emph{Current Persona} block; empty outputs fall back to the previous
    step's persona.
    \item \textbf{\texttt{no\_persona}.} The persona returned by the model is
    discarded, so step $n{+}1$ receives the cold-start placeholder (``No prior
    observations yet; build the persona from scratch''). Every step is
    effectively a cold start from the model's perspective.
\end{itemize}
All other prompt components are byte-identical between the two conditions; the
only varying token span is the \emph{Current Persona} block, which isolates
the contribution of carried persona from every other inference factor.

\paragraph{Models reported.}
Figure~\ref{fig:persona_stepwise} highlights DeepSeek-V3, MiniMax-M2.5 and
Gemini-3-Flash as representatives of three capability tiers (moderate,
moderate-strong, and frontier).

\subsection{RQ2: Long-Context Baseline vs.\ Streaming}                            
  \label{app:rq2_design}                   
                                                                                    
  \paragraph{Baseline design.}
  The \texttt{long-context} baseline differs from the streaming pipeline only       
  in how the user history is presented to the model:
  \begin{itemize}
      \item \textbf{\texttt{streaming}.} At the final step $N$ the model reads
      the current batch $B_N$ conditioned on the persona $P_{N-1}$ carried from
      prior steps; each earlier step is processed independently and its
      \texttt{persona\_summary} is chained forward.
      \item \textbf{\texttt{long-context}.} The model receives a single
      concatenation of \emph{all} past posts separated by date headers
      (``\texttt{--- \textlangle date\textrangle~---}'') and emits one JSON
      prediction in a single shot, without any intermediate persona.
  \end{itemize}
  The candidate pool $C$, selection size
  $k = \max(1, \mathrm{round}(0.25\cdot|C|))$, distractor meta-fields, and
  ground truth $T^{\text{new}}\cup T^{\text{keep}}$ are all taken from the
  user's \emph{last} recorded step, and the streaming pipeline's prediction at
  that step is re-scored against the same split --- so both methods predict the
  same tags on the same users.

  \paragraph{Granularity groups.}
  To investigate how iteration depth affects streaming quality, we define three
  batching granularities that control the number of streaming steps per user:
  \begin{itemize}
      \item \textbf{Fine} (\mytilde7 steps): threshold $t{=}3$ for high-frequency
      platforms, $t{=}1$ for low-frequency platforms.
      \item \textbf{Default} (\mytilde5 steps): threshold $t{=}5$ / $t{=}3$.
      \item \textbf{Coarse} (\mytilde3 steps): threshold $t{=}8$ / $t{=}5$.
  \end{itemize}
  Finer granularity yields more iteration steps with smaller batches per step,
  while coarser granularity compresses more posts per step in fewer iterations.
  The long-context baseline serves as a zero-iteration reference point that
  processes all raw posts at once.

  \paragraph{User filter.}
  Evaluation is restricted to users with at least four observed steps
  ($\texttt{MIN\_STEPS}=4$), excluding very short trajectories for which the
  ``compression vs.\ raw'' distinction is vacuous.

  \paragraph{Context-length check.}
  The concatenated-history prompt fits comfortably inside the context window of
  every tested model (Gemini-3-Flash, DeepSeek-V3, MiniMax-M2.5): its
  95th-percentile token count is well below each model's advertised limit. No
  truncation or sliding window is applied, so the streaming advantage observed
  in Table~\ref{tab:streaming_vs_longctx} reflects the benefit of iterative
  persona refinement rather than a capacity workaround.

%% file: custom.bib
@inproceedings{hu2024ladder,
  title={Ladder-of-thought: Using knowledge as steps to elevate stance detection},
  author={Hu, Kairui and Yan, Ming and Chong, Wen Haw and Yap, Yong Keong and Guan, Cuntai and Zhou, Joey Tianyi and Tsang, Ivor W},
  booktitle={2024 International Joint Conference on Neural Networks (IJCNN)},
  pages={1--8},
  year={2024},
  organization={IEEE}
}

@inproceedings{mu2024navigating,
  title={Navigating prompt complexity for zero-shot classification: A study of large language models in computational social science},
  author={Mu, Yida and Wu, Ben P and Thorne, William and Robinson, Ambrose and Aletras, Nikolaos and Scarton, Carolina and Bontcheva, Kalina and Song, Xingyi},
  booktitle={Proceedings of the 2024 Joint International Conference on Computational Linguistics, Language Resources and Evaluation (LREC-COLING 2024)},
  pages={12074--12086},
  year={2024}
}

@article{kheiri2023sentimentgpt,
  title={Sentimentgpt: Exploiting gpt for advanced sentiment analysis and its departure from current machine learning},
  author={Kheiri, Kiana and Karimi, Hamid},
  journal={arXiv preprint arXiv:2307.10234},
  year={2023}
}

@article{wu2024understanding,
  title={Understanding the role of user profile in the personalization of large language models},
  author={Wu, Bin and Shi, Zhengyan and Rahmani, Hossein A and Ramineni, Varsha and Yilmaz, Emine},
  journal={arXiv preprint arXiv:2406.17803},
  year={2024}
}

@inproceedings{shi2025personax,
  title={PersonaX: A recommendation agent-oriented user modeling framework for long behavior sequence},
  author={Shi, Yunxiao and Xu, Wujiang and Zeqi, Zhang and Zi, Xing and Wu, Qiang and Xu, Min},
  booktitle={Findings of the Association for Computational Linguistics: ACL 2025},
  pages={5764--5787},
  year={2025}
}

@article{zhao2025nextquill,
  title={Nextquill: Causal preference modeling for enhancing llm personalization},
  author={Zhao, Xiaoyan and You, Juntao and Zhang, Yang and Wang, Wenjie and Cheng, Hong and Feng, Fuli and Ng, See-Kiong and Chua, Tat-Seng},
  journal={arXiv preprint arXiv:2506.02368},
  year={2025}
}

@inproceedings{qiu2025measuring,
  title={Measuring what makes you unique: Difference-aware user modeling for enhancing llm personalization},
  author={Qiu, Yilun and Zhao, Xiaoyan and Zhang, Yang and Bai, Yimeng and Wang, Wenjie and Cheng, Hong and Feng, Fuli and Chua, Tat-Seng},
  booktitle={Findings of the Association for Computational Linguistics: ACL 2025},
  pages={21258--21277},
  year={2025}
}

@article{ren2026alpbench,
  title={ALPBench: A Benchmark for Attribution-level Long-term Personal Behavior Understanding},
  author={Ren, Lu and She, Junda and Luo, Xinchen and Wang, Tao and Ye, Xin and Zhang, Xu and Wang, Muxuan and Yang, Xiao and Wang, Chenguang and Xie, Fei and others},
  journal={arXiv preprint arXiv:2602.03056},
  year={2026}
}

@article{bartkowiak2025edgewisepersona,
  title={EdgeWisePersona: A Dataset for On-Device User Profiling from Natural Language Interactions},
  author={Bartkowiak, Patryk and Podstawski, Michal},
  journal={arXiv preprint arXiv:2505.11417},
  year={2025}
}

@article{liu2025itdr,
  title={ITDR: An Instruction Tuning Dataset for Enhancing Large Language Models in Recommendations},
  author={Liu, Zekun and Huang, Xiaowen and Sang, Jitao},
  journal={arXiv preprint arXiv:2508.05667},
  year={2025}
}

@article{prottasha2025user,
  title={User profile with large language models: Construction, updating, and benchmarking},
  author={Prottasha, Nusrat Jahan and Kowsher, Md and Raman, Hafijur and Anny, Israt Jahan and Bhat, Prakash and Garibay, Ivan and Garibay, Ozlem},
  journal={arXiv preprint arXiv:2502.10660},
  year={2025}
}

@inproceedings{geng2022recommendation,
  title={Recommendation as language processing (rlp): A unified pretrain, personalized prompt \& predict paradigm (p5)},
  author={Geng, Shijie and Liu, Shuchang and Fu, Zuohui and Ge, Yingqiang and Zhang, Yongfeng},
  booktitle={Proceedings of the 16th ACM conference on recommender systems},
  pages={299--315},
  year={2022}
}

@article{gao2023chat,
  title={Chat-rec: Towards interactive and explainable llms-augmented recommender system},
  author={Gao, Yunfan and Sheng, Tao and Xiang, Youlin and Xiong, Yun and Wang, Haofen and Zhang, Jiawei},
  journal={arXiv preprint arXiv:2303.14524},
  year={2023}
}

@inproceedings{bao2023tallrec,
  title={Tallrec: An effective and efficient tuning framework to align large language model with recommendation},
  author={Bao, Keqin and Zhang, Jizhi and Zhang, Yang and Wang, Wenjie and Feng, Fuli and He, Xiangnan},
  booktitle={Proceedings of the 17th ACM conference on recommender systems},
  pages={1007--1014},
  year={2023}
}

@inproceedings{wang2025lettingo,
  title={Lettingo: Explore user profile generation for recommendation system},
  author={Wang, Lu and Zhang, Di and Yang, Fangkai and Zhao, Pu and Liu, Jianfeng and Zhan, Yuefeng and Sun, Hao and Lin, Qingwei and Deng, Weiwei and Zhang, Dongmei and others},
  booktitle={Proceedings of the 31st ACM SIGKDD Conference on Knowledge Discovery and Data Mining V. 2},
  pages={2985--2995},
  year={2025}
}

@inproceedings{lu2025prompt,
  title={Prompt tuning as user inherent profile inference machine},
  author={Lu, Yusheng and Du, Zhaocheng and Li, Xiangyang and Jia, Pengyue and Wang, Yejing and Liu, Weiwen and Wang, Yichao and Guo, Huifeng and Tang, Ruiming and Dong, Zhenhua and others},
  booktitle={Proceedings of the 34th ACM International Conference on Information and Knowledge Management},
  pages={5898--5906},
  year={2025}
}

@article{liu2025recoworld,
  title={Recoworld: Building simulated environments for agentic recommender systems},
  author={Liu, Fei and Lin, Xinyu and Yu, Hanchao and Wu, Mingyuan and Wang, Jianyu and Zhang, Qiang and Zhao, Zhuokai and Xia, Yinglong and Zhang, Yao and Li, Weiwei and others},
  journal={arXiv preprint arXiv:2509.10397},
  year={2025}
}

@article{piao2025agentsociety,
  title={Agentsociety: Large-scale simulation of llm-driven generative agents advances understanding of human behaviors and society},
  author={Piao, Jinghua and Yan, Yuwei and Zhang, Jun and Li, Nian and Yan, Junbo and Lan, Xiaochong and Lu, Zhihong and Zheng, Zhiheng and Wang, Jing Yi and Zhou, Di and others},
  year={2025}
}

@article{wang2023recagent,
  title={Recagent: A novel simulation paradigm for recommender systems},
  author={Wang, Lei and Zhang, Jingsen and Chen, Xiaowen and Lin, Yankai and Song, Ruihua and Zhao, Wayne Xin and Wen, Ji-Rong},
  journal={arXiv preprint arXiv:2306.02552},
  year={2023}
}

@article{jiang2025know,
  title={Know me, respond to me: Benchmarking llms for dynamic user profiling and personalized responses at scale},
  author={Jiang, Bowen and Hao, Zhuoqun and Cho, Young-Min and Li, Bryan and Yuan, Yuan and Chen, Sihao and Ungar, Lyle and Taylor, Camillo J and Roth, Dan},
  journal={arXiv preprint arXiv:2504.14225},
  year={2025}
}

@article{zhao2026inside,
  title={Inside Out: Evolving User-Centric Core Memory Trees for Long-Term Personalized Dialogue Systems},
  author={Zhao, Jihao and Chen, Ding and Fan, Zhaoxin and Xu, Kerun and Hu, Mengting and Tang, Bo and Xiong, Feiyu and Li, Zhiyu},
  journal={arXiv preprint arXiv:2601.05171},
  year={2026}
}

@article{li2025conf,
  title={Conf-Profile: A Confidence-Driven Reasoning Paradigm for Label-Free User Profiling},
  author={Li, Yingxin and Zhao, Jianbo and Ren, Xueyu and Tang, Jie and You, Wangjie and Chen, Xu and Zhou, Kan and Feng, Chao and Ran, Jiao and Meng, Yuan and others},
  journal={arXiv preprint arXiv:2509.18864},
  year={2025}
}

@inproceedings{sabouri2025towards,
  title={Towards explainable temporal user profiling with LLMs},
  author={Sabouri, Milad and Mansoury, Masoud and Lin, Kun and Mobasher, Bamshad},
  booktitle={Adjunct Proceedings of the 33rd ACM Conference on User Modeling, Adaptation and Personalization},
  pages={219--227},
  year={2025}
}

@inproceedings{qi2021hierec,
  title={HieRec: Hierarchical user interest modeling for personalized news recommendation},
  author={Qi, Tao and Wu, Fangzhao and Wu, Chuhan and Yang, Peiru and Yu, Yang and Xie, Xing and Huang, Yongfeng},
  booktitle={Proceedings of the 59th Annual Meeting of the Association for Computational Linguistics and the 11th International Joint Conference on Natural Language Processing (Volume 1: Long Papers)},
  pages={5446--5456},
  year={2021}
}

@article{purificato2024user,
  title={User modeling and user profiling: A comprehensive survey},
  author={Purificato, Erasmo and Boratto, Ludovico and De Luca, Ernesto William},
  journal={arXiv preprint arXiv:2402.09660},
  year={2024}
}

@article{tan2023user,
  title={User modeling in the era of large language models: Current research and future directions},
  author={Tan, Zhaoxuan and Jiang, Meng},
  journal={arXiv preprint arXiv:2312.11518},
  year={2023}
}

@inproceedings{shi2025you,
  title={You are what you bought: Generating customer personas for e-commerce applications},
  author={Shi, Yimin and Fei, Yang and Zhang, Shiqi and Wang, Haixun and Xiao, Xiaokui},
  booktitle={Proceedings of the 48th International ACM SIGIR Conference on Research and Development in Information Retrieval},
  pages={1810--1819},
  year={2025}
}

@article{singh2025openai,
  title={Openai gpt-5 system card},
  author={Singh, Aaditya and Fry, Adam and Perelman, Adam and Tart, Adam and Ganesh, Adi and El-Kishky, Ahmed and McLaughlin, Aidan and Low, Aiden and Ostrow, AJ and Ananthram, Akhila and others},
  journal={arXiv preprint arXiv:2601.03267},
  year={2025}
}

@article{hurst2024gpt,
  title={Gpt-4o system card},
  author={Hurst, Aaron and Lerer, Adam and Goucher, Adam P and Perelman, Adam and Ramesh, Aditya and Clark, Aidan and Ostrow, AJ and Welihinda, Akila and Hayes, Alan and Radford, Alec and others},
  journal={arXiv preprint arXiv:2410.21276},
  year={2024}
}

@misc{minimax2025m25,
  author       = {{MiniMax}},
  title        = {MiniMax-M2.5 Model},
  year         = {2025},
  publisher    = {huggingface},
  url          = {https://huggingface.co/MiniMaxAI/MiniMax-M2.5},
}

@misc{zaiorg2025glm47,
  author       = {{ZAI}}, 
  title        = {GLM-4.7 Model},
  year         = {2025},
  publisher    = {huggingface},
  url          = {https://huggingface.co/zai-org/GLM-4.7},
}

@article{liu2025deepseek,
  title={Deepseek-v3. 2: Pushing the frontier of open large language models},
  author={Liu, Aixin and Mei, Aoxue and Lin, Bangcai and Xue, Bing and Wang, Bingxuan and Xu, Bingzheng and Wu, Bochao and Zhang, Bowei and Lin, Chaofan and Dong, Chen and others},
  journal={arXiv preprint arXiv:2512.02556},
  year={2025}
}

@article{grattafiori2024llama,
  title={The llama 3 herd of models},
  author={Grattafiori, Aaron and Dubey, Abhimanyu and Jauhri, Abhinav and Pandey, Abhinav and Kadian, Abhishek and Al-Dahle, Ahmad and Letman, Aiesha and Mathur, Akhil and Schelten, Alan and Vaughan, Alex and others},
  journal={arXiv preprint arXiv:2407.21783},
  year={2024}
}

@misc{deepmind2025gemini3flash,
  author       = {{Google DeepMind}},
  title        = {Gemini 3 Flash Model Card},
  year         = {2025},
  url          = {https://storage.googleapis.com/deepmind-media/Model-Cards/Gemini-3-Flash-Model-Card.pdf},
}

@article{yang2025qwen3,
  title={Qwen3 technical report},
  author={Yang, An and Li, Anfeng and Yang, Baosong and Zhang, Beichen and Hui, Binyuan and Zheng, Bo and Yu, Bowen and Gao, Chang and Huang, Chengen and Lv, Chenxu and others},
  journal={arXiv preprint arXiv:2505.09388},
  year={2025}
}

@article{agarwal2025gpt,
  title={gpt-oss-120b \& gpt-oss-20b model card},
  author={Agarwal, Sandhini and Ahmad, Lama and Ai, Jason and Altman, Sam and Applebaum, Andy and Arbus, Edwin and Arora, Rahul K and Bai, Yu and Baker, Bowen and Bao, Haiming and others},
  journal={arXiv preprint arXiv:2508.10925},
  year={2025}
}

@inproceedings{salemi2024lamp,
  title={Lamp: When large language models meet personalization},
  author={Salemi, Alireza and Mysore, Sheshera and Bendersky, Michael and Zamani, Hamed},
  booktitle={Proceedings of the 62nd Annual Meeting of the Association for Computational Linguistics (Volume 1: Long Papers)},
  pages={7370--7392},
  year={2024}
}

@article{zhao2025llms,
  title={Do llms recognize your preferences? evaluating personalized preference following in llms},
  author={Zhao, Siyan and Hong, Mingyi and Liu, Yang and Hazarika, Devamanyu and Lin, Kaixiang},
  journal={arXiv preprint arXiv:2502.09597},
  year={2025}
}

@inproceedings{tan2025personabench,
  title={Personabench: Evaluating ai models on understanding personal information through accessing (synthetic) private user data},
  author={Tan, Juntao and Yang, Liangwei and Liu, Zuxin and Liu, Zhiwei and RN, Rithesh and Awalgaonkar, Tulika Manoj and Zhang, Jianguo and Yao, Weiran and Zhu, Ming and Kokane, Shirley and others},
  booktitle={Findings of the Association for Computational Linguistics: ACL 2025},
  pages={878--893},
  year={2025}
}

@article{jiang2025personamem,
  title={Personamem-v2: Towards personalized intelligence via learning implicit user personas and agentic memory},
  author={Jiang, Bowen and Yuan, Yuan and Shen, Maohao and Hao, Zhuoqun and Xu, Zhangchen and Chen, Zichen and Liu, Ziyi and Vijjini, Anvesh Rao and He, Jiashu and Yu, Hanchao and others},
  journal={arXiv preprint arXiv:2512.06688},
  year={2025}
}

@article{eke2019survey,
  title={A survey of user profiling: State-of-the-art, challenges, and solutions},
  author={Eke, Christopher Ifeanyi and Norman, Azah Anir and Shuib, Liyana and Nweke, Henry Friday},
  journal={IEEE Access},
  volume={7},
  pages={144907--144924},
  year={2019},
  publisher={IEEE}
}
